\documentclass[review,12pt]{elsarticle}
\usepackage[margin=1in]{geometry}
\usepackage{xcolor}
\usepackage{amsfonts}
\usepackage{amsmath}
\usepackage{amsthm}
\usepackage{bm}
\usepackage{optidef}
\usepackage{setspace}

\usepackage{tikz}
\usetikzlibrary{hobby,shapes,arrows}

\usepackage{hyperref}
\hypersetup{colorlinks=true,allcolors=blue}

\usepackage{soul}

\usepackage[linesnumbered,ruled,vlined]{algorithm2e}

\SetCommentSty{mycommfont}

\makeatletter
\newcommand{\algorithmfootnote}[2][\footnotesize]{%
  \let\old@algocf@finish\@algocf@finish
  \def\@algocf@finish{\old@algocf@finish
    \leavevmode\rlap{\begin{minipage}{\linewidth}
    #1#2
    \end{minipage}}%
  }%
}
\makeatother

\newtheorem{lemma}{Lemma}
\newtheorem{theorem}{Theorem}
\newtheorem{definition}{Definition}
\newtheorem{assumption}{Assumption}
\newtheorem{corollary}{Corollary}

\newtheorem{remark}{Remark}
\newtheorem{appclaim}{Claim A\ignorespaces}

\DeclareMathOperator*{\argmin}{arg\,min}

\begin{document}

\begin{frontmatter}
    \title{GenMod: A generative modeling approach for spectral representation of PDEs with random inputs}

    \author[1]{Jacqueline Wentz}

    \author[1]{Alireza Doostan\texorpdfstring{\corref{cor1}}{}}
    \ead{alireza.doostan@colorado.edu}

    \address[1]{Aerospace Engineering Sciences Department, University of Colorado, Boulder, CO 80309, USA}

    \cortext[cor1]{Corresponding author}

    \begin{abstract}
        We propose a method for quantifying uncertainty in high-dimensional PDE systems with random parameters, where the number of solution evaluations is small. Parametric PDE solutions are often approximated using a spectral decomposition based on polynomial chaos expansions. For the class of systems we consider (i.e., high dimensional with limited solution evaluations) the coefficients are given by an underdetermined linear system in a regression formulation. This implies additional assumptions, such as sparsity of the coefficient vector, are needed to approximate the solution. Here, we present an approach where we assume the coefficients are close to the range of a generative model that maps from a low to a high dimensional space of coefficients. Our approach is inspired be recent work examining how generative models can be used for compressed sensing in systems with random Gaussian measurement matrices. Using results from PDE theory on coefficient decay rates, we construct an explicit generative model that predicts the polynomial chaos coefficient magnitudes. The algorithm we developed to find the coefficients, which we call GenMod, is composed of two main steps. First, we predict the coefficient signs using Orthogonal Matching Pursuit. Then, we assume the coefficients are within a sparse deviation from the range of a sign-adjusted generative model. This allows us to find the coefficients by solving a nonconvex optimization problem, over the input space of the generative model and the space of sparse vectors. We obtain theoretical recovery results for a Lipschitz continuous generative model and for a more specific generative model, based on coefficient decay rate bounds. We examine three high-dimensional problems and show that, for all three examples, the generative model approach outperforms sparsity promoting methods at small sample sizes.
    \end{abstract}

    \begin{keyword}
        uncertainty quantification; polynomial chaos expansion; compressed sensing; nonlinear approximation; sparse approximation
    \end{keyword}
\end{frontmatter}

\section{Introduction}

The goal of uncertainty quantification (UQ) is to predict how uncertain or random parameters influence a quantity of interest (QoI). Here, we consider UQ for physical systems that are modeled using partial differential equations (PDEs), where the QoI could, for example, be the steady-state solution at a specific point in space. Uncertainty in this system can result from either natural variability or from a lack of knowledge on the precise input parameter values. We are interested in scenarios where the physical system is complex and the number of uncertain parameters is large. Due to the curse of dimensionality, the number of samples needed at a specific resolution grows rapidly with the dimension (i.e., with the number of uncertain parameters). The complexity of the system, and hence the PDE, makes it computationally infeasible to perform the required number of simulations at this large number of input vector realizations. In this paper, we develop a novel method for performing UQ for this class of systems.

For approximating stochastic PDE solutions, a common UQ approach is a spectral method based on polynomial chaos (PC) expansions \cite{Wiener1938,Ghanem1991,Xiu2002,Xiu2003a}. Suppose that $\bm{Y}\in\mathbb{R}^d$ represents a vector of $d$ random parameters  with probability density function $\rho(\bm{Y})$ and assume that the scalar-valued QoI, given as $u(\bm{Y})$, has finite variance. The PC expansion of the QoI is written as the sum of coefficients $\{c_j\}$ multiplied by orthogonal polynomials $\{\psi_j\}$,
\begin{equation}\label{eq:intro-sum}
    u(\bm{Y}) = \sum_{j=1}^\infty c_j \psi_j(\bm{Y}) \approx \sum_{j=1}^P c_j \psi_j(\bm{Y}).
\end{equation}
Since $u(\bm{Y})$ has finite variance, the infinite sum given by (\ref{eq:intro-sum}) converges, and a finite truncation of the series, say of the first $P$ terms with a suitable ordering of $\{\psi_j\}$, provides an approximate solution. Note that the probability density function $\rho(\bm{Y})$ defines how to sample the uncertain parameters as well as the specific form of the orthogonal polynomials. For example, the uniform density distribution implies uniform sampling and the use of Legendre polynomials \cite{Xiu2002}.

Our goal is to estimate the coefficients $\{c_j\}$ using the available QoI evaluations. Then, using (\ref{eq:intro-sum}), we can estimate the QoI at arbitrary values of the random input vector. Specifically, given $N$ realizations of $\bm{Y}$, i.e., $\bm{y}^{(i)} \in  \mathbb{R}^d$ for $i=1,\dots, N$, and the corresponding, possibly noisy, QoI evaluation vector, $\bm{u} = [u(\bm{y}^{(1)}),u(\bm{y}^{(2)}),\dots,u(\bm{y}^{(N)})]^T$, we find the coefficient vector $\bm{\hat{c}}=[\hat{c}_1,\hat{c}_2,\dots,\hat{c}_P]^T$ that minimizes the distance between the approximate solution, $\hat{\bm{u}} = \Psi \bm{\hat{c}}$, and the observed solution $\bm{u}$. Here, $\Psi \in \mathbb{R}^{N\times P}$ is the measurement matrix that contains the orthogonal polynomial evaluations at the realizations of the random input, i.e., $\Psi_{i,j} = \psi_j(\bm{y}^{(i)})$.

For high dimensional PDEs with limited measurements, the linear system $\bm{u} = \Psi \bm{\hat{c}}$ may be underdetermined, i.e. $N<P$, and, therefore, additional assumptions are needed to approximate the coefficient vector. For specific classes of PDEs, the magnitude of the coefficients show exponential decay as the degree of the contributing polynomial increases \cite{Cohen2010,Beck2012,Tran2017}. This implies that the coefficient vector is compressible \cite{Candes2006}, and hence, can be approximated using a sparsity assumption \cite{Doostan2011,Rauhut2012,Chkifa2017} and compressed sensing \cite{Cohen2008}. However, this assumption does not take advantage of the known structure behind the exponential decay. Not only do higher order polynomials likely contribute less, but the rate of exponential decay can vary based on the specific direction considered in the random input space (e.g., see Proposition 3.1 from \cite{Beck2012}). Methods exist for enforcing a structured sparsity by limiting the possible support sets of the signal \cite{Baraniuk2010}, using alternative basis selection methods \cite{Chkifa2017}, or via a weighted $\ell_1$ minimization \cite{Peng2014,Chkifa2017}. However, these methods do not directly enforce structured exponential decay of the coefficients when it exists, and, instead, continue to rely on the assumption of sparsity in the coefficient vector.

We propose estimating the PC coefficients using a generative model that takes into account the decaying structure of the coefficient vector. For our purposes, we refer to a generative model as a model or function that maps from a lower to a higher dimensional space, i.e. $G:\mathbb{R}^k \rightarrow \mathbb{R}^P$ where $k<P$. Often, the term ``generative model" refers to a statistical model that is trained to create an output (e.g., image of a face) given a random input drawn from a probability distribution. For instance, in machine learning two well known types of generative models that use neural networks include variational auto-encoders (VAEs) \cite{Kingma2014} and generative adversarial networks (GANs) \cite{Goodfellow2020}. In contrast to these trainable generative models, here we define and evaluate a generative model that does not require initial training but is rather based on the predicted structure of the coefficient vector. Notably, the true coefficient vector may not precisely follow the structure defined by the generative model. However, given the limited number of solution evaluations, our approach provides a feasible method for regularizing the regression problem. Additionally, to account for scenarios where the generative model does not adequately capture the coefficient structure, we allow for sparse deviations from the range of the generative model.

The idea of using a generative model to approximate the signal vector in compressed sensing was recently introduced \cite{Bora2017}. To the authors' knowledge, this idea has only been applied to systems with a random Gaussian measurement matrix, $A\in\mathbb{R}^{N\times P}$. In this context, the goal is to find the signal $\bm{x} \in \mathbb{R}^P$ such that $A\bm{x} = \bm{b}$ where the observation vector $\bm{b}\in\mathbb{R}^N$ is known and the signal $\bm{x}$ is in the range of a generative model $G:\mathbb{R}^k \rightarrow \mathbb{R}^P$ (i.e., $\exists \bm{z} \in \mathbb{R}^k$ such that $G(\bm{z})=\bm{x}$). In a later work, the assumption that $\bm{x}$ is in the range of $G$ is relaxed by allowing for sparse deviations \cite{Dhar2018}. For the class of systems considered by \cite{Bora2017,Dhar2018}, recovery is guaranteed with high probability given the sample size exceeds a threshold value. The proof of this result relies on showing that the random Gaussian matrix $A$ satisfies the Set-Restricted Eigenvalue Condition (S-REC), which is a generalization of the Restricted Eigenvalue Condition (REC) \cite{Bickel2009}.
Note that in \cite{Bora2017,Dhar2018}, the generative models explored numerically were GANs and VAEs.

\subsection{Our contribution}

We extend the concept of compressed sensing with generative models to systems where the measurement matrix contains orthogonal polynomial evaluations, allowing for the recovery of structured signals. This allows for the generative modeling approach to be applied to perform UQ. We specifically focus on the Legendre measurement matrix and provide recovery results for a general $L$-Lipschitz function $G$. Along the way we derive Johnson-Lindenstrauss like distributional bounds for the Legendre measurement matrix. We note that our theoretical recovery results are strongly based on the previous work done by \cite{Bora2017,Dhar2018}. Specifically, we show that the Legendre measurement matrix satisfies the S-REC.

We derive an explicit generative model that is based on known bounds for the PC coefficients. The basic model predicts that, as the degree of the contributing polynomial increases, the magnitude of the corresponding coefficient decays exponentially. We use a separate method, based on Orthogonal Matching Pursuit (OMP) \cite{Pati1993}, to predict the coefficient signs, as these are not given by the generative model. Note that other methods, such as those based on $\ell_1$ minimization, could also be used to help determine the coefficient signs. Although we are biasing the system towards having a specific decay structure, our approach allows for deviations from this structure in multiple ways. First, the generative model includes features other then exponential decay, and, second, we allow for sparse deviations from the range of the sign-adjusted generative model, i.e., the generative model output multiplied by the predicted coefficient signs. Including sparse deviations also allows for coefficients with incorrect signs to flip.

The goal behind the generative model approach is to increase approximation accuracy at low sample sizes. Because the dimension of the latent space (i.e., the space that the generative model acts on) does not increase with the sample size, we expect that at some threshold sample size our approach will no longer outperform sparsity promoting methods such as those based on $\ell_1$ minimization or OMP \cite{Doostan2011,Yang2013,Peng2014}. Our numerical results show that, given the sample size is small enough, the generative model approach does outperform these other methods.

The structure of this paper is outlined as follows. In Section~\ref{sec:methods} we present an overview of the types of PDE systems under consideration (Section~\ref{sec:methods-pde}) and how the solution to these PDEs can be approximated using PC expansions (Section~\ref{sec:methods-pce}). We provide an explicit equation for a generative model that is based on the decay of the PC coefficients (Section~\ref{sec:methods-genmod}) and describe the optimization algorithm that is used to find the PC coefficient vector (Section~\ref{sec:methods-alg}). In Section~\ref{sec:previous-theory}, we provide an overview of previous theoretical results that will be directly referred to in our proofs. In Section~\ref{sec:theory-results}, we present the main theoretical results of the paper, and, in Section~\ref{sec:proofs}, we provide the proofs. Note that in some instances proofs are moved to \ref{sec:appB}.  In Section~\ref{sec:numerical-results}, we present the numerical results for three example problems. We demonstrate that the generative model outperforms iteratively-reweighted Lasso \cite{Tibshirani1996,Candes2008} and OMP for each of these problems at small sample sizes. Finally, in Section~\ref{sec:discussion}, we conclude with a discussion of possible future improvements.

\section{Methods}\label{sec:methods}

\subsection{PDE}\label{sec:methods-pde}

Let $\bm{Y}=(Y_1,Y_2,\dots,Y_d)$ represent a $d$-dimensional random vector defined on the complete probability space $(\Gamma, \mathcal{F}, \mathcal{P})$ where $\Gamma = \Gamma_1 \times\dots\times\Gamma_d$. We  assume that the random variables $Y_i$ are independent and uniformly distributed on $\Gamma_i = [-1,1]$ for $i=1,\dots,d$. We consider scenarios where the random vector $\bm{Y}$ represents uncertain parameters (e.g., boundary conditions, the diffusion coefficient, physical constants) that influence the solution $u(\bm{x},\bm{Y})$ of a PDE. We will use $\bm{\alpha} \in (\mathbb{N} \cup \{0\})^d$ to represent a multi-index vector that corresponds to the $d$-dimensional random input space. Let $\bm{r} \in \mathbb{R}^d$ and define $\bm{r}^{\bm{\alpha}}=\sum_{i=1}^N r_i^{\alpha_i}$, $|\bm{\alpha}|=\sum_{i=1}^d \alpha_i$, and $\bm{\alpha}! = \sum_{i=1}^d \alpha_i!$.

For a class of elliptic PDEs, there exist known, exponentially decreasing, bounds on the PC coefficients used to describe the solution of the PDE \cite{Beck2012}. Specifically, consider the following stochastic system defined on a convex bounded polygonal domain $\mathcal{D} \in \mathbb{R}^D$ with boundary $\partial \mathcal{D}$
\begin{equation}\label{eq:PDE-intro}
    \begin{aligned}
        -\nabla \cdot (a(\bm{x},\bm{Y})\nabla u(\bm{x},\bm{Y})) &= f(\bm{x}), & \bm{x} &\in \mathcal{D} \\
        u(\bm{x},\bm{Y}) &= 0, & \bm{x} &\in \partial \mathcal{D},
    \end{aligned}
\end{equation}
where we make the following assumptions about the  diffusion coefficient $a(\bm{x},\bm{Y})$:
\begin{assumption}[see Assumption 2.1 from \mbox{\cite{Beck2012}}]
    There exists $a_{min}>0$ and $a_{max}<\infty$ such that
    \begin{equation}
        \mathbb{P}(a_{min} \le a(\bm{x},\bm{Y})\le a_{max}, \forall\bm{x}\in \bar{\mathcal{D}})=1,
    \end{equation}
    where $\mathbb{P}(\cdot)$ represents the probability of an event.
\end{assumption}
\begin{assumption}[see Assumption 2.3 from \mbox{\cite{Beck2012}}]\label{assump:2}
    The diffusion coefficient $a(\bm{x},\bm{Y})$ is infinitely many times differentiable with respect to $\bm{Y}$ and $\exists \bm{r} \in \mathbb{R}_+^d$ s.t. for all $\bm{Y} \in \Gamma$
    \begin{equation}
        \left\|\frac{\partial_{\bm{\alpha}} a}{a}(\cdot,\bm{Y})\right\|_{L^\infty(D)} \le \bm{r}^{\bm{\alpha}} \quad \text{ with } \quad \partial_{\bm{\alpha}} a = \frac{\partial^{\alpha_1 + \dots + \alpha_d} a}{\partial y_1^{\alpha_1}\dots\partial y_d^{\alpha_d}},
    \end{equation}
    where $\bm{\alpha} \in (\mathbb{N} \cup \{0\})^d$ is a multi-index and $\bm{r}$ is independent of $\bm{y}$.
\end{assumption}
\noindent
In Section~{\ref{sec:methods-genmod}}, we present results showing that given $\bm{r}$ satisfies additional conditions, there exist known bounds on the PC coefficients. Note that, for systems where the diffusion coefficient is given by an expansion of the form $a(\bm{x},\bm{Y}) = a_0 + \sum b_i(\bm{x}) Y_i$, Assumption {\ref{assump:2}} is satisfied with $\bm{r} = [r_1,r_2,\dots,r_d]$, where $r_i = \|b_i\|_{L^\infty(D)}/a_{min}$ \mbox{\cite{Beck2012}}.

For the example applications, we consider a 1D version of (\ref{eq:PDE-intro}) in Section~\ref{sec:Ex1}, whereas in Section~\ref{sec:Ex23} we consider a more complex PDE of a heat driven cavity flow problem.

\subsection{Polynomial chaos expansion}\label{sec:methods-pce}

To approximate the solution of a PDE with uniformly distributed inputs, we use the Legendre PC basis functions. Let $\{\psi_j\}$ represent the set of univariate orthonormal Legendre polynomials of degree $j$, which are normalized such that,
\begin{equation}
    \int_{[-1,1]} \psi_j^2(y)\rho(y)dy=1,
\end{equation}
where the probability measure is $\rho = 1/2$.

Define the set $\Lambda_{p,d} := \{\bm{\alpha} \in (\mathbb{N} \cup \{0\})^d \mid |\bm{\alpha}|\le p \}$ to contain the multi-indices that correspond to the $d$-dimensional Legendre polynomials with degree of at most $p$. Note that the cardinality of $\Lambda_{d,p}$ is
\begin{equation}\label{eq:P}
    P := |\Lambda_{d,p}| = \frac{(p+d)!}{p!d!}.
\end{equation}

Using the multi-indices given in $\Lambda_{p,d}$, the solution to a PDE is approximated as
\begin{equation}
    u(\bm{x},\bm{Y}) = \sum_{\alpha \in (\mathbb{N}\cup\{0\})^d} c_{\bm{\alpha}}(\bm{x}) \psi_{\bm{\alpha}}(\bm{Y}) \approx \sum_{\bm{\alpha} \in \Lambda_{p,d}} c_{\bm{\alpha}}(\bm{x}) \psi_{\bm{\alpha}}(\bm{Y}).
\end{equation}
Here, each $d$-dimensional Legendre polynomial $\psi_{\bm{\alpha}}(\bm{Y})$ is equal to the product of univariate Legendre polynomials with degrees that are defined by the multi-index $\bm{\alpha}$, i.e.,
\begin{equation}
    \psi_{\bm{\alpha}}(\bm{Y}) = \prod_{i=1}^d \psi_{\alpha_i} (Y_i).
\end{equation}

We will order the multi-indices $\Lambda_{p,d} = \{\bm{\alpha}^{(1)},\bm{\alpha}^{(2)},\dots,\bm{\alpha}^{(P)}\}$ according to the following rules. Multi-indices with smaller $\ell_1$ norms appear first, i.e., if $\|\bm{\alpha}^{(i)}\|_1 < \|\bm{\alpha}^{(j)}\|_1$ then $i<j$. If the $\ell_1$ norm is equal, multi-index vectors are sorted such that those with larger values at lower indices appear first. That is, if $\alpha_k^{(i)} > \alpha_k^{(j)}$ for $k=\min\{\ell \mid \alpha_{\ell}^{(i)} \ne \alpha_{\ell}^{(j)}\}$, then $i<j$. Using this ordering and suppressing the dependency of $u$ on $\bm{x}$, we write the approximate solution as follows,
\begin{equation}
    u(\bm{Y}) = \sum_{i=1}^P c_i \psi_{\bm{\alpha}^{(i)}}(\bm{Y}).
\end{equation}

Suppose we have $N<P$ realizations of the random vector $\bm{Y}$, denoted as $\bm{y}^{(i)}$ for $i=1,\dots,N$ and the corresponding, possibly noisy, solution evaluation vector, $\bm{u} =  [u(\bm{y}^{(1)}),u(\bm{y}^{(2)}),\dots,u(\bm{y}^{(N)})]^T$. We define the Legendre measurement matrix $\Psi \in \mathbb{R}^{N \times P}$, elementwise, as
\begin{equation}
    \Psi_{ij} = \psi_{\bm{\alpha}^{(j)}} (\bm{y}^{(i)}).
\end{equation}
Each row of $\Psi$ corresponds to a sample realization and each column corresponds to one of the $P$ multi-indices. Our goal is to find the approximate coefficient vector $\bm{\hat{c}} = [\hat{c}_1,\hat{c}_2,\dots,\hat{c}_P]$ such that $\Psi \bm{\hat{c}} \approx \bm{u}$.

\subsection{Generative model based on coefficient decay}\label{sec:methods-genmod}

Our goal is to approximate the coefficient vector as the output of a generative model plus a sparse vector \cite{Bora2017,Dhar2018}. To define the generative model, we use known structural characteristics of the coefficient vector for PC expansions of PDE solutions. For specific classes of PDEs, it is known that the coefficient vector is compressible, i.e., the vector features a rapid decay of coefficient amplitude such that $|c_{\mathcal{I}(i)}| \le C i^{-1/r}$, where $C$ and $r$ are positive constants and $\mathcal{I}(i)$ indexes the sorted coefficients, i.e.,  $|c_{\mathcal{I}(1)}|>|c_{\mathcal{I}(2)}|>\dots>|c_{\mathcal{I}(P)}|$~\cite{Candes2006,Cohen2010,Beck2012,Tran2017}.
The compressibility of the signal implies that a sparsity assumption will lead to accurate signal recovery via compressed sensing \cite{Doostan2011,Rauhut2012,Chkifa2017}. Here, we will leverage the compressible structure of the coefficient vector to generate a nonlinear decay model that maps from a lower to a higher dimensional coefficient space. 

Theoretical work has shown that, under certain conditions, the Legendre PC coefficients of (\ref{eq:PDE-intro}) decay exponentially, where the decay rate varies depending on the direction considered in the random input space. Specifically, from Proposition 3.1 in \cite{Beck2012} we have that
\begin{equation}\label{eq:c-bound}
    \|c_{\bm{\alpha}}\|_{H_0^1(D)}
    \le C_0 \frac{|\bm{\alpha}|!}{\bm{\alpha}!}e^{-\sum_i g_i \alpha_i},
    \quad g_i = -\log(r_i/(\sqrt{3}\log2)),
\end{equation}
where $C_0 > 0$ is a constant, $\bm{r} \in\mathbb{R}_+^d$ is as given in Assumption~\ref{assump:2}, and $\|\cdot\|_{H_0^1(D)}$ is the gradient norm, i.e., $\|c\|_{H_0^1(D)} = \|\nabla c\|_{L^2(D)}$. Clearly, if $r_i < \sqrt{3}\log 2$ for $i=1,\dots,d$, then the coefficients exhibit exponential decay. Although (\ref{eq:c-bound}) provides a useful starting point for predicting coefficient values, the coefficient magnitudes could show different trends within the bounding envelope. Additionally, (\ref{eq:c-bound}) does not provide information on the coefficient signs.

To predict the magnitude of the coefficients, we construct a generative model using a generalized form of (\ref{eq:c-bound}). This model allows for other characteristics besides exponential decay to exist but guarantees, as the polynomial degree increases, exponential decay will be the dominate effect. Specifically, the generative model is defined as follows: let $G:\mathbb{R}^{2d+1} \rightarrow \mathbb{R}^P$ with the $i$th element of $G$ given as
\begin{equation}\label{eq:G}
    G_i(C,\bm{g},\bm{h}) := C \frac{|\bm{\alpha}^{(i)}|!}{\bm{\alpha}^{(i)}!}\prod_{j=1}^d (1+\alpha_j^{(i)})^{h_j} e^{-g_j \alpha_j^{(i)}},
\end{equation}
where $C\in\mathbb{R}_+$, $\bm{g}\in\mathbb{R}_+^d$, and $\bm{h}\in\mathbb{R}^d$ represent the $2d+1$ parameters. Recall that $d$ represents the number of uncertain variables, and, thus, the dimension of the latent space $2d+1$, increases linearly, rather than exponentially, with the number of input parameters. We consider scenarios where $d$ is large, but due to the exponential dependency of $P$ on $d$ (see (\ref{eq:P})) we have that $2d+1 \ll P$. This implies that $G$ maps from a lower to a higher dimensional space as desired.
In contrast to (\ref{eq:c-bound}), the generative model given by (\ref{eq:G}) contains an algebraic growth/decay term which allows for the coefficients of low order polynomials to show different characteristics besides exponential decay.

\begin{remark}
    Alternative functions for $G$ could be explored that are still motivated by the concept of exponential decay (e.g., the summation of decaying terms).  However, the purpose of this work was not to find the ideal function $G$, but rather to demonstrate the utility of using a generative model with exponential and algebraic decay qualities.
\end{remark}

Because the generative model only represents the magnitude of the coefficients, we define the vector $\bm{\zeta} \in \{-1,1\}^P$ to contain the signs of the $P$ coefficients. The final coefficient vector is then given as
\begin{equation}\label{eq:c}
    \bm{c} = D_{\bm{\zeta}} G(\bm{z}) + \bm{\nu},
\end{equation}
where the $2d+1$ parameters are now contained in the vector $\bm{z}$, i.e. $\bm{z}=(C,\bm{g},\bm{h})$, $D_{\bm{\zeta}} = \text{diag}(\bm{\zeta}) \in \mathbb{R}^{P \times P}$, and the vector $\bm{\nu} \in \mathbb{R}^P$ is assumed to be sparse. The values of $\bm{z}$, $\bm{\zeta}$, and $\bm{\nu}$ are determined using the GenMod algorithm, i.e., Algorithm~\ref{alg}, which is presented in the next section.

\begin{remark}
    In previous work, using generative models for compressed sensing, the sparse vector $\bm{\nu}$ was included to allow the signal to deviate from the range of the generative model \cite{Dhar2018}. In the present work, adding the vector $\bm{\nu}$ has the additional benefit of fixing coefficient signs that are initially labeled incorrectly during the optimization procedure.
\end{remark}

\subsection{Optimization problem and algorithm}\label{sec:methods-alg}

Given $N$ realizations of the random vector $Y$,  $\bm{y}^{(i)}$ for $i=1,\dots,N$, and the resulting noisy QoI evaluation vector $\bm{u}=[u(\bm{y}^{(1)}),u(\bm{y}^{(2)}),\dots,u(\bm{y}^{(N)})]^T$, our goal is to find the true coefficient vector $\bm{c}^*$ such that
\begin{equation}
    \bm{u} = \Psi \bm{c}^* + \bm{\eta},
\end{equation}
where $\Psi\in \mathbb{R}^{N\times P}$ is the Legendre measurement matrix calculated using the $N$ realizations of $\bm{Y}$ and $\bm{\eta}\in\mathbb{R}^N$ represents possible measurement noise. As mentioned, we assume that $\bm{c}^*$ can be approximated using ({\ref{eq:c}}).

This results in a challenging optimization problem because, to find the coefficient vector, we must find the input vector $\bm{z}\in\mathbb{R}^{2d+1}$, the sparse vector $\bm{\nu}\in\mathbb{R}^P$, and the coefficient signs $\bm{\xi}\in\{-1,1\}^P$. We propose using a sequential approach, where we first predict $\bm{\xi}$ using OMP and then we find $\bm{z}$ and $\bm{\nu}$ by iteratively searching for a solution to the following optimization problem:
\begin{mini}
    {\substack{\bm{z}\in\mathbb{R}^{2d+1}, \bm{\nu}\in\mathbb{R}^P}}
    {L(\bm{z},\bm{\nu};\bm{\zeta},\Psi,\bm{u}) + \lambda \| W(\bm{z}) \bm{\nu}\|_1,}
    {\label{opt:1}}
    {}
\end{mini}
where
\begin{equation}\label{eq:L}
        L(\bm{z},\bm{\nu};\bm{\zeta},\Psi,\bm{u}) :=
        \|\Psi (D_{\bm{\zeta}}G(\bm{z})+\bm{\nu}) - \bm{u}\|_2^2
\end{equation}
and $W(\bm{z})$ is a diagonal weight matrix such that
\begin{equation}
    W_{ii}(\bm{z}) := \frac{1}{G_i(\bm{z}) + \epsilon}.
\end{equation}
We are weighting the vector $\bm{\nu}$ based on the generative model's prediction of the coefficient magnitude. This implies that the penalty of adjusting a coefficient's value is scaled by the generative model's prediction. Thus, if $\bm{\nu}$ causes a sign to flip, the penalty is equal amongst all coefficients. This helps in scenarios where the sign of a coefficient is incorrectly determined in $\bm{\zeta}$ and the coefficient magnitude predicted by the generative model is large.

Algorithm~\ref{alg} provides the details of the GenMod optimization procedure. In the description of the algorithm we use the following definition,
\begin{equation}\label{eq:DeltaL}
    \Delta L(\bm{z}^{(1)},\bm{z}^{(2)},\bm{\nu};\bm{\zeta},\Psi,\bm{u}) :=
        \frac{L(\bm{z}^{(1)},\bm{\nu};\bm{\zeta},\Psi,\bm{u})-L(\bm{z}^{(2)},\bm{\nu};\bm{\zeta},\Psi,\bm{u})}{L(\bm{z}^{(2)},\bm{\nu};\bm{\zeta},\Psi,\bm{u})}.
\end{equation}
Prior to calling the algorithm, the $N$ training data points are divided into $N_{op}$ points used for optimization and $N_{va}$ points used for validation. The algorithm takes as input measurement matrices that correspond to these two datasets (i.e., $\Psi_{op} \in \mathbb{R}^{N_{op} \times P}$ and $\Psi_{va} \in \mathbb{R}^{N_{va} \times P}$) as well as the corresponding observations (i.e., $\bm{u}_{op} \in \mathbb{R}^{N_{op}}$ and $\bm{u}_{va} \in \mathbb{R}^{N_{va}}$). For some steps of the algorithm these two datasets are combined, as described below.

As mentioned the algorithm has two main stages. In the first stage, we set the predicted signs of the coefficients, i.e. $\bm{\zeta} \in \mathbb{R}^P$, using OMP (see lines 5-10 of Algorithm~\ref{alg}). In the second stage, we approximate a solution to (\ref{opt:1}) (see lines 11-25 or Algorithm~\ref{alg}). One study considered a similar optimization problem and solved it using gradient descent \cite{Dhar2018}, noting that the cost function is non-differentiable at only a finite number of points. For our system we found that this approach led to instabilities and, therefore, in Algorithm~\ref{alg} we use a different method. Specifically, we propose an alternating approach where, at each iteration, we first use Adam gradient descent \cite{Kingma2015} to find
\begin{equation}
    \bm{z}^{(i)} = \argmin_{\bm{z}} L(\bm{z},\bm{\nu}^{(i-1)};\bm{\zeta},\Psi,\bm{u}).
\end{equation}
We then use weighted Lasso \cite{Tibshirani1996} to find
\begin{equation}
    \bm{\nu}^{(i)} = \argmin_{\bm{\nu}} \left(L(\bm{z}^{(i)},\bm{\nu};\bm{\zeta},\Psi,\bm{u}) + \lambda\|W^{(i)} \bm{\nu}\|_1\right),
\end{equation}
where
\begin{equation}
   W_{jj}^{(i)} = \frac{1}{G_j(\bm{z}^{(i)}) + \epsilon G_0 (\bm{z}^{(i)})}.
\end{equation}
For the algorithm we set $\epsilon=1e-4$. We repeat this process until a convergence criteria is met. Specifically, we stop the iterations if the loss function $L$ does not decrease when applied to the validation data.

In Algorithm~\ref{alg}, we keep the optimization and validation data separated for performing Adam optimization, but we combine these two datasets when performing OMP and Lasso. Each time Adam optimization is performed, we use the optimization dataset to train the model and the validation dataset to select the best point for performing the next iteration of Lasso. For the OMP and Lasso steps, the optimization and validation datasets are combined in order to use a cross-validation approach.  We implement OMP using OrthogonalMatchingPursuitCV from the scikit-learn python package \cite{Pedregosa2011}. OMP is performed using a $k$-fold cross-validation with $k=5$ folds.
For the Lasso steps, we again use a cross-validation approach, with $k=5$ folds, and select the value of $\lambda$ using the ``one-standard-error" rule \cite{Hastie2009}; see Algorithm~\ref{alg:Lasso}. This algorithm is similar to the standard form of Lasso with $k$-fold cross validation, but with an increased bias towards sparsity. Specifically, rather then picking the $\lambda$ that minimizes the mean reconstruction error, i.e., $\lambda_L$ in Algorithm~{\ref{alg:Lasso}}, we pick the maximum $\lambda$ that is within one standard error of $\lambda_L$. Within the implementation of this algorithm we use the LassoCV function from the scikit-learn python package.

\begin{algorithm}
    \DontPrintSemicolon
    \SetNoFillComment
    \caption{GenMod($\Psi_{op}$,$\bm{u}_{op}$,$\Psi_{va}$,$\bm{u}_{va}$).}
    \label{alg}
    $\Psi =$ [$\Psi_{op}$;  $\Psi_{va}$] \;
    $\bm{u} =$ [$\bm{u}_{op}$; $\bm{u}_{va}$] \;
    $\bm{z}^{(0)} = \left(\log\left(\left|\frac{1}{N}\sum_{i=1}^N u_i\right|\right),\bm{1}_{2d}\right)$ \;
    $\bm{\nu}^{(0)} = \bm{0}_P$ \;
    \tcc{Generate initial prediction of coefficient signs}
    $\tilde{\bm{c}}$ = OrthogonalMatchingPursuitCV($\Psi$,$\bm{u}$,folds=5) \;
    \uIf{$\tilde{c}_i \ne 0$}
    {
        $\zeta_i = \text{sign}(\tilde{c}_i)$
    }
    \uElse
    {
        $\zeta_i = \text{sign}((\Psi^{(i)})^T(\bm{u}-\Psi\tilde{c}))$
    }
    $\bm{\zeta} = [\zeta_1,\zeta_2,\dots,\zeta_P]^T$ \;
    \tcc{Alternate Adam and Lasso}
    \For{$k=1,2,\dots,\text{max\_iteration}$}
    {
        Set $\bm{z}^{(k,0)} = \bm{z}^{(k-1)}$ \;
        \tcc{Run Adam Optimization}
        $\bm{m} = \bm{0}_k$ \;
        $\bm{v} = \bm{0}_k$ \;
        \For{$t=1,2,\dots,\text{max\_adam\_iteration}$}
        {
            $\bm{z}^{(k,t)},\bm{m},\bm{v} =$ AdamStep($\bm{z}^{(k,t-1)}$,$\bm{\nu}^{(k-1)}$,$\bm{\zeta}$,$\Psi_{op}$,$\bm{u}_{op}$,$\bm{m}$,$\bm{v}$,$t$) \;
            \tcc{Enforce exponential decay}
            \If{$\bm{z}_i^{(k,t)} < 0$ for $i=2,\dots,d+1$}
            {
                $\bm{z}_i^{(k,t)} = 0$
            }
            Break if $\Delta L(\bm{z}^{(k,t)},\bm{z}^{(k,t-1)},\bm{\nu}^{(k-1)};\bm{\zeta},\Psi_{op},\bm{u}_{op}) < \varepsilon$\;
        }
        \tcc{Find best $\bm{z}$ given validation data}
        Set $\mathcal{Z}^{(k)} = \{\bm{z}^{(k,0)},\bm{z}^{(k,1)},\dots\}$ \;
        $\bm{z}^{(k)} = \argmin_{\bm{z} \in \mathcal{Z}^{(k)}} L(\bm{z},\bm{\nu}^{(k-1)};\bm{\zeta},\Psi_{va},\bm{u}_{va})$ \;
        \tcc{Run weighted lasso with cross validation}
        $\bm{x}$ = LassoWithStErRule($\Psi (W(\bm{z}^{(k)}))^{-1}$, $\bm{u} - \Psi D_{\bm{\zeta}} G(\bm{z}^{(k)})$;$\bm{\lambda}$)\;
        $\bm{\nu}^{(k)} = (W(\bm{z}^{(k)}))^{-1}\bm{x}$ \;
        \tcc{If no improvement on validation data, stop}
        \If{$L(\bm{z}^{(k)},\bm{\nu}^{(k)};\bm{\zeta},\Psi_{va},\bm{u}_{va}) > L(\bm{z}^{(k-1)},\bm{\nu}^{(k-1)};\bm{\zeta},\Psi_{va},\bm{u}_{va})$}
        {
            Break
        }
    }
    \Return{$\bm{z}^{(k)}$,$\bm{\nu}^{(k)}$,$\bm{\xi}$}
    \algorithmfootnote{The vector $\bm{1}_{2d} \in \mathbb{R}^{2d}$ is composed of ones and $\bm{0}_P \in \mathbb{R}^P$ is composed of zeros. Additionally, for the examples considered in Section~{\ref{sec:numerical-results}}, we set $\varepsilon=1e-6$ and the $\bm{\lambda}$ vector in the LassoWithStErRule() is constructed automatically by the scikit learn LassoCV function. For details on AdamStep() see \ref{sec:appA} or Algorithm 1 in \cite{Kingma2015}.}
\end{algorithm}

\begin{algorithm}
    \DontPrintSemicolon
    \SetNoFillComment
    \caption{LassoWithStErRule($\Phi$,$\bm{u}$;$\bm{\lambda}$).}
    \label{alg:Lasso}
    Split data $\bm{u}$ and matrix $\Phi$ into $k$ folds;  $\bm{u} = [\bm{u}^{(1)},\dots,\bm{u}^{(5)}]$ and $\Phi = [\Phi^{(1)};\dots;\Phi^{(k)}]$ \;
    Let $\bm{u}^{(j^C)}$ and $\Phi^{(j^C)}$ denote the data and matrix with the $j$th fold removed \;
    $N_{\lambda} = \text{size}(\bm{\lambda})$ \;
    \For{$\ell = 1,\dots,N_{\lambda}$} {
        \For{$j=1,\dots,k$} {
            \tcc{Perform Lasso using all data except the $j$th fold.}
            $\bm{x}$ = Lasso($\Phi^{(j^C)}$,$\bm{u}^{(j^C)}$,$\lambda$=$\lambda_{\ell}$) \;
            \tcc{Calculate the reconstruction error using the $j$th fold.}
            $e_{\ell,j} = \|\Phi^{(j)}\bm{x} - \bm{u}^{(j)}\|_2^2$
        }
        \tcc{Find the mean reconstruction error across the folds}
        $e_{\ell} = \frac{1}{k} \sum_{i=1}^{k} e_{\ell,i}$
    }
    \tcc{Find the final hyperparameter using the standard error rule}
    $L= \argmin_{\ell} e_{\ell}$ \;
    $s_{L}$ = $\frac{1}{\sqrt{N}} \text{std}([e_{L,1},e_{L,2},\dots,e_{L,k}])$ \;
    $\lambda_{final} = \text{maximum } \lambda_{\ell}$ such that $e_{\ell} < e_{L} + s_{L}$. \;
    \tcc{Perform Lasso using the final hyperparameter value on the full dataset}
    $\bm{\nu} =$ Lasso($\Phi,\bm{u},\lambda_{final}$). \;
    \Return{$\bm{\nu}$}
\end{algorithm}
For the remainder of the paper, we assume that Algorithm~\ref{alg} provides a close to optimal coefficient vector. We do not prove this result, but in practice we find the coefficient vectors predicted by Algorithm~\ref{alg} are consistently accurate (see Section~\ref{sec:numerical-results}).

\subsection{Lasso and Orthogonal Matching Pursuit}\label{sec:methods-las-omp}

We will compare our optimization algorithm to approaches where we use iteratively-reweighted (IRW) Lasso \cite{Candes2008} or OMP to directly find the coefficient vector. In both these approaches the underlying assumption is that the coefficient vector is sparse. As in Algorithm {\ref{alg}}, to perform  OMP we use the scikit-learn python package and perform $k$-fold cross validation using $k=5$ folds. Algorithm {\ref{alg:IRW-Lasso}} in \ref{sec:appA} describes the steps of the IRW Lasso approach. The goal of this algorithm is to find a solution to the following optimization problem:
\begin{mini}
    {\substack{\bm{c}\in\mathbb{R}^P}}
    {\|\Psi \bm{c} - \bm{u}\| + \lambda \| W \bm{c}\|_1,}
    {\label{opt:IRW-Lasso}}
    {}
\end{mini}
where $W$ is a diagonal matrix such that
\begin{equation}
    W_{ii} = \frac{1}{|c_i| + \tau}.
\end{equation}
Note, in contrast to previous implementations of IRW Lasso, if a convergence criteria is not met, we repeat the algorithm at increasing values of $\tau$. In our numerical examples, we set the initial value of $\tau=10^{-4}$ and found that in no cases did we need to increase $\tau$ greater than $\tau=10^{-1}$ to obtain convergence.

\subsection{Notational notes}

For the remainder of the paper we will use $A \in\mathbb{R}^{N\times P}$ to refer to an arbitrary measurement matrix and $\Psi \in\mathbb{R}^{N\times P}$ to refer to the more specific Legendre measurement matrix. We also define the matrix $\Phi = \frac{1}{\sqrt{N}} \Psi D_{\bm{\xi}}$, which will be used extensively in the theoretical results section. We define the vector $\bm{\xi} \in \{-1,1\}^P$ to represent a Rademacher sequence \mbox{\cite{Wainwright-Ch2-2019}}, i.e., uniformly distribute on $\{-1,1\}^P$ and set $D_{\bm{\xi}} = \text{diag}(\bm{\xi})$. Finally, we recall that the probability of an event $x$ is written as $\mathbb{P}(x)$.

\section{Previous Theoretical Results}\label{sec:previous-theory}

In this section we discuss previous work that is used in Section~\ref{sec:theory-results} to prove our main theoretical results. In Section~\ref{sec:previous-theory-1} we present details from studies that used generative models for compressed sensing \cite{Bora2017,Dhar2018}. The theory behind this work relies on showing that a random measurement matrix satisfies certain distributional properties with high probability. Therefore, in Section~\ref{sec:previous-theory-2}, we discuss background information on these properties in the context of the Legendre measurement matrix.

\subsection{Compressed sensing using generative models}\label{sec:previous-theory-1}

Whether a generative model leads to accurate signal recovery depends on both the measurement matrix $A$ and the form of the generative model \cite{Bora2017}. Suppose $S \subseteq \mathbb{R}^P$ is equal to the range of the generative model. Then, we are guaranteed accurate signal recovery if the measurement matrix $A$ satisfies the Set-Restricted Eigenvalue Condition (S-REC), first defined by \cite{Bora2017}, on $S$.
\begin{definition}
    Let $S \subseteq \mathbb{R}^P$. For $\gamma >0, \delta \ge 0$, a matrix $A \in \mathbb{R}^{N\times P}$ satisfies the S-REC($S,\gamma,\delta$), if $\forall \bm{x}_1,\bm{x}_2\in S$,
    \begin{equation}
        \|A(\bm{x}_1-\bm{x}_2)\|_2 \ge \gamma \|\bm{x}_1 - \bm{x}_2\|_2 - \delta.
    \end{equation}
\end{definition}

The random Gaussian matrix satisfies the S-REC($S$,$\gamma$,$\delta$) for specific sets $S$, if the number of measurements $N$ exceeds a threshold value dependent on $\gamma$ and $\delta$ \cite{Bora2017,Dhar2018}. Specifically, this result applies when $A \in \mathbb{R}^{N \times P}$ where $A_{ij} \sim \mathcal{N}(0,\frac{1}{N})$ and when $S$ is either the range of $G$ (see Lemma 4.1, \cite{Bora2017}) or the set of sparse deviations from the range of $G$ (see Lemma~2 from \cite{Dhar2018}).
To the authors' knowledge, the S-REC condition has not been studied for measurement matrices other than the random Gaussian matrix. That is, this is the  first paper that has studied the S-REC as applied to the Legendre measurement matrix.

If a matrix $A \in\mathbb{R}^{N\times P}$ satisfies the S-REC on the set $S$, then the solution to
\begin{mini}
    {\bm{x} \in S}
    {\|A \bm{x} - \bm{u}\|_2^2}
    {\label{opt:3}}
    {}
\end{mini}
is close to the true solution, $\bm{x}^* \in\mathbb{R}^P$, of the linear problem $\bm{u}=A \bm{x}^*+\bm{\eta}$, where $\bm{\eta}$ represents measurement noise. This is stated formally in the following lemma.
\begin{lemma}[Lemma 4.3 from \cite{Bora2017}]\label{lemma:bora-result}
    Let $A \in \mathbb{R}^{N\times P}$ be drawn from a distribution that (1) satisfies the S-REC($S,\gamma,\delta$) with probability at least $1-\mu$ and (2) for every fixed $\bm{x} \in \mathbb{R}^P$ has $\|A \bm{x}\|_2\le2\|\bm{x}\|_2$ with probability at least $1-\mu$. For any $\bm{x}^* \in \mathbb{R}^P$ and noise $\bm{\eta}$, let $\bm{u}=A \bm{x}^*+\bm{\eta}$. Let $\hat{\bm{x}}$ approximately minimize $\|\bm{u}-A \bm{x}\|_2$ over $\bm{x}\in S$, i.e., for some $\epsilon>0$,
    \begin{equation}
        \|\bm{u}-A\hat{\bm{x}}\|_2\le \min_{\bm{x}\in S} \|\bm{u}-A \bm{x}\|_2 +\epsilon.
    \end{equation}
    Then,
    \begin{equation}\label{eq:err-bound-bora}
        \|\hat{\bm{x}}-\bm{x}^*\|_2 \le \left(\frac{4}{\gamma} + 1\right) \min_{\bm{x} \in S} \|\bm{x}^* - \bm{x}\|_2 + \frac{2\|\bm{\eta}\|_2}{\gamma} + \frac{\epsilon}{\gamma}+ \frac{\delta}{\gamma}\
    \end{equation}
    with probability at least $1-2\mu$.
\end{lemma}

The error bound given by (\ref{eq:err-bound-bora}) contains four terms: The first is error caused by $S$ not containing $\bm{x}^*$ (i.e., error when the generative model does not approximate the solution well), the second is error from measurement noise, the third is error due to the optimization algorithm not finding the optimal solution, and the fourth is error caused by the slack term $\delta$ in the S-REC definition.

\begin{remark}
    Although the approximation error given by ({\ref{eq:err-bound-bora}}) scales with the noise magnitude $\|\bm\eta\|_2$, this error term will not grow as we increase the number of measurements $N$ for the examples considered here, i.e., the Gaussian random matrix $A$ where $A_{ij} \sim \mathcal{N}\left(0,\frac{1}{N}\right)$ and the Legendre measurement matrix $\frac{1}{\sqrt{N}}\Psi$. This is because both these matrices have a $1/\sqrt{N}$ scaling.
\end{remark}

To apply Lemma~\ref{lemma:bora-result} to the PC expansion system, we will show that $\Phi$, a variation of the Legendre measurement matrix, satisfies the S-REC and that for $\bm{x}\in\mathbb{R}^P$, $\|\Phi \bm{x} \|_2 \le 2\|\bm{x}\|_2$. To do this, we use an approach similar to methods presented in \cite{Bora2017,Dhar2018}. These studies proved the random Gaussian matrix satisfies the S-REC by using known concentration tail inequalities. For the random Gaussian matrix these concentration inequalities are as follows. If $A$ is such that $A_{ij} \sim \mathcal{N}(0,\frac{1}{N})$, then $Y = N \frac{\|A\bm{x}\|_2^2}{\|\bm{x}\|_2^2}$ follows a chi-squared distribution with $N$ degrees of freedom and, therefore, $Y$ is sub-exponential, i.e.,
\begin{equation}
    \mathbb{P}(|Y - \mathbb{E}(Y)| \ge t) \le
        \begin{cases}
            2e^{-\frac{t^2}{2\nu^2}} & 0 \le t \le \nu^2/b \\
            2e^{-\frac{t}{2b}} & t > \nu^2/b,
        \end{cases}
\end{equation}
where $(\nu,b)=(2\sqrt{N},4)$ \cite{Wainwright-Ch2-2019}. Setting $t = \epsilon N$ and rearranging we have that $A$ satisfies the following concentration inequalities,
\begin{equation}\label{eq:A-conc-bound}
    \mathbb{P}\left(\left|\frac{\|A\bm{x}\|_2^2}{\|\bm{x}\|_2^2} - 1\right| \ge \epsilon\right) \le
        \begin{cases}
            2e^{-\frac{N\epsilon^2}{8}} & 0 \le \epsilon \le 1\\
            2e^{-\frac{N\epsilon}{8}} & \epsilon > 1.
        \end{cases}
\end{equation}
In the present work our goal is to show that that the Legendre measurement matrix satisfies a probability bound similar to (\ref{eq:A-conc-bound}). Importantly, this bound must be valid for large values of $\epsilon$.

\subsection{The restricted isometry and Johnson-Lindenstrauss properties}\label{sec:previous-theory-2}

To show the Legendre measurement matrix $\Psi$ satisfies the necessary concentration inequalities, we will leverage previous results. Specifically, it is known that $\frac{1}{\sqrt{N}}\Psi$ satisfies the restricted isometry property (RIP) \cite{Peng2014} and, given a matrix satisfies the RIP, certain distributional tail inequalities follow with high probability \cite{Krahmer2011}.

The RIP is defined as follows:
\begin{definition}
    We say a matrix $A\in\mathbb{R}^{N\times P}$ satisfies the RIP($s$,$\delta$) if
    \begin{equation}
        (1-\delta) \|\bm{x}\|_2^2 \le \|A \bm{x}\|_2^2 \le (1+\delta)\|\bm{x}\|_2^2
    \end{equation}
    for any vector $\bm{x} \in \mathbb{R}^P$ such that $\|\bm{x}\|_0 \le s$.
\end{definition}
The following corollary states that, given a sufficient sample size, the Legendre measurement matrix satisfies the RIP with high probability. This is a modified version of Corollary 3.1 from \cite{Peng2014} based on the original results for bounded orthonormal systems \cite{Rauhut2010}.

\begin{corollary}\label{corollary:RIP}
    Let $\Psi\in\mathbb{R}^{N\times P}$ be a Legendre measurement matrix. If
    \begin{equation}\label{eq:RIP-Psi}
        N \ge C 3^p \delta^{-2} s \log^2(s)\log^2(P),
    \end{equation}
    then $\frac{1}{\sqrt{N}}\Psi$ satisfies the RIP($s$,$\delta$) with probability at least than $1-e^{-\gamma\log^2(s)\log^2(P)}$, where $C$ and $\gamma$ are constants independent of $N$ and $p$.
\end{corollary}

\begin{remark}
    Corollary~{\ref{corollary:RIP}} gives a slightly different result when compared to Corollary 3.1 from \mbox{\cite{Peng2014}}. The result given in \mbox{\cite{Peng2014}} relies on the assumption that $s \ge 3^p \delta^{-2} \log(P)$ (see \mbox{\cite{Foucart2010}}). In order to avoid this assumption, we consider the more general bound given by \mbox{\cite{Rauhut2010}},
    \begin{equation}
        \frac{N}{\log(N)} \ge C 3^p \delta^{-2} s \log^2(s) \log(P),
    \end{equation}
    and note that ({\ref{eq:RIP-Psi}}) implies this bound, assuming that $N<P$.
\end{remark}

\begin{remark}
    Typically, the restricted isometry constant $\delta$ is assumed to be such that $\delta \in (0,1)$. However, in later proofs, we allow for $\delta > 1$. The result in Corollary~\ref{corollary:RIP} is still valid; however, modifications to the constants, as presented in \cite{Rauhut2010}, might be necessary.
\end{remark}

Since $\frac{1}{\sqrt{N}}\Psi$ satisfies the RIP, we will use relationships between the RIP and Johnson-Lindenstrauss (JL) distributional property to obtain concentration inequalities similar to (\ref{eq:A-conc-bound}). The JL distribution property is defined as follows:
\begin{definition}
    We say a matrix $A \in \mathbb{R}^{N\times P}$ satisfies the JL distributional property at level $\epsilon$, i.e., JL($\epsilon$), if for any $x \in \mathbb{R}^P$
    \begin{equation}
        (1-\epsilon)\|\bm{x}\|_2^2 \le \|A \bm{x}\|_2^2 \le (1+\epsilon)\|\bm{x}\|_2^2.
    \end{equation}
\end{definition}
Given the matrix $A$ satisfies the RIP, the following theorem gives the probability that the matrix $AD_{\bm{\xi}}$ satisfies the JL distributional property, where we recall $\bm{\xi} \in \mathbb{R}^P$ is a Rademacher sequence.
\begin{theorem}[Theorem 3.1 from \cite{Krahmer2011}]\label{thm:JL}
    Fix $\eta > 0$ and $\epsilon \in (0,1)$, and consider a finite set $E \subset \mathbb{R}^P$ of cardinality $|E| = e$. Set $s \ge 40 \log \frac{4e}{\eta}$ to be an even integer and suppose $A \in \mathbb{R}^{N\times P}$ satisfies the RIP of order $s$ and level $\delta \le \frac{\epsilon}{4}$. Let $\bm{\xi} \in \mathbb{R}^P$ be a Rademacher sequence. Then with probability exceeding $1-\eta$
    \begin{equation}
        (1-\epsilon)\|\bm{x}\|_2^2 \le \|A D_{\bm{\xi}} \bm{x}\|_2^2 \le (1+\epsilon)\|\bm{x}\|_2^2
    \end{equation}
    uniformly for all $\bm{x} \in E$.
\end{theorem}
If we consider a single point $\bm{x}\in\mathbb{R}^P$ and set $E=\{\bm{x}\}$ in Theorem~{\ref{thm:JL}}, then the theorem implies the following statement:
Let $s$ be an even integer and set $\eta:=4e^{-c_0s}$ where $c_0 = 1/40$. If $A$ satisfies the RIP($s$,$\epsilon/4$), then for a given $\bm{x}\in\mathbb{R}^P$,
\begin{equation}\label{eq:JL-given-RIP}
    \begin{aligned}
        \mathbb{P}\left(\left|\|A D_{\bm{\xi}} \bm{x}\|_2^2 - \|\bm{x}\|_2^2\right| \ge \epsilon \|\bm{x}\|_2^2 \right) &\le \eta = 4e^{-c_0 s}.
    \end{aligned}
\end{equation}
Note that here $\eta$ is chosen to be the minimum value for which the inequality $s\ge40\log\frac{4e}{\eta}$ is satisfied.

This result provides a useful starting point for showing that the Legendre measurement matrix satisfies a probability bound similar to (\ref{eq:A-conc-bound}), but additional work is still required. Theorem~\ref{thm:JL} assumes a matrix deterministically satisfies the RIP, but the Legendre measurement matrix only satisfies the RIP with a specific probability. Additionally, the bound on the probability given in (\ref{eq:JL-given-RIP}) depends on the sparsity level $s$ at which the RIP is satisfied. We instead need this bound to depend on $\epsilon$ and the number of measurements $N$. In Proposition 3.2 from \cite{Krahmer2011}, they obtain such a bound for a specific class of systems. These results provide a useful starting point for the present work.

\section{Main Theoretical Results}\label{sec:theory-results}
\subsection{Theoretical results for general generative model}

We present theoretical results detailing the number of measurements required to accurately recover the coefficient vector for the Legendre measurement matrix when using a generative model approach. For high-dimensions, where the number of basis elements $P$ is sufficiently large, the number of samples $N$ must satisfy
\begin{equation}
    N = \mathcal{O}\left(3^p\log^4(P)\left(k\log\frac{Lr}{\delta} + \ell \log(P)\right) \right).
\end{equation}
Here, $L$ is the Lipschitz constant of the generative model $G$, $r$ is the minimum radius of a $k$-dimensional ball that contains the domain of $G$,  i.e. if $G:\Omega \rightarrow \mathbb{R}^P$ then $\Omega \subseteq B_k(r)$, and $\delta$ is the slack term the S-REC definition. Recall that $k$ represents the dimension the latent space that the generative model acts on, and $\ell$ is the sparsity of the vector $\bm{\nu}$ this is added to the output of the generative model.

For simplification we present the following definitions which will be used throughout this section. First we define the constants: $c_0=1/40$, $c_1=16C$, $c_2 = 0.9$, $c_3=c_2/c_1$, $c_4=16$, and $C$ and $\gamma$ are the constants given in Corollary~\ref{corollary:RIP}. Next define the function $g$ as follows:
\begin{equation}\label{eq:g}
    g(k,\ell,P,\delta,\alpha,L,r) := \frac{3}{\alpha^2}
            \left(
                k\log\left(\frac{4 L r}{\delta}\right)
                + \frac{2\ell+1}{2}\log\left(\frac{eP(2\ell+1)}{2\ell\alpha}\right)
            \right).
\end{equation}
Additionally, define $s_0$ implicitly as the largest real number that satisfies,
\begin{equation}\label{eq:s0}
    c_0 s_0 - \gamma \log^2(s_0) \log^2(P) = 0.
\end{equation}

\begin{remark}
    Since $c_0,\gamma > 0$, (\ref{eq:s0}) has two positive solutions. For the smaller solution, either $s_0<1$ or
    \begin{equation}
        s_0 \le \exp\left(\left(\frac{c_0 s_0}{\gamma \log^2(P)}\right)^{1/2}\right) \le \exp\left(\left(\frac{c_0}{\gamma \log^2(P)}\right)^{1/2}\right) < 2
    \end{equation}
    assuming $\log(P) > 2\sqrt{c_0/\gamma}$. This condition is satisfied, since $c_0/\gamma < 1$. Therefore, there is only one solution to (\ref{eq:s0}) such that $s_0 \ge 2$, which will be the solution of interest here.
\end{remark}

Finally, for $\epsilon>0$ we define
\begin{equation}\label{eq:Neps}
    N_{\epsilon} := c_1 \epsilon^{-2}3^p s_0\log^2(s_0)\log^2(P).
\end{equation}
Later in  this section we will see that $N_{\epsilon}$ represents a threshold sample size at which a JL like probability bound for the Legendre measurement matrix changes form.

The main recovery result is stated formally in the following theorem.
\begin{theorem}\label{thm:1}
    Let $\Omega \subseteq B^k(r) = \{\bm{z} \in \mathbb{R}^k \mid \|\bm{z}\|_2 < r\}$ and let $G:\Omega \rightarrow \mathbb{R}^P$ be $L$-Lipschitz. Define the set $S \subseteq \mathbb{R}^P$ such that
    \begin{equation}
        S=\{ G(\bm{z}) + \bm{\nu} \mid \bm{z} \in \Omega, \|\bm{\nu}\|_0 \le \ell\}.
    \end{equation}
    Using $A \in \mathbb{R}^{N \times P}$, $\bm{c}^* \in \mathbb{R}^P$, and $\bm{\eta} \in \mathbb{R}^N$, let $\bm{u} = A \bm{c}^* + \bm{\eta}$.

    For $\alpha < 1/4$ and $\delta < 4Lr/10$ , let $g=g(k,\ell,P,\delta,\alpha,L,r)$ be as defined in (\ref{eq:g}) and $s_0$ be as defined implicitly in (\ref{eq:s0}). Suppose $P$ is large enough, such that
    \begin{equation}
        s_0 > \frac{1}{c_0 c_1 c_3}(\log(5) + 4Jk + g)
    \end{equation}
    and $N$ is large enough, such that,
    \begin{equation}
        N \ge \frac{3^{p}s_0}{\gamma c_3} g.
    \end{equation}
    Let $\tilde{N} = \min\{N,N_{\epsilon_{max}}\}$ where $N_{\epsilon_{max}}$ is given by (\ref{eq:Neps}) with
    \begin{equation}
        \epsilon^2_{max} := \max\left\{\frac{1}{g}(\log(5)+4Jk)+1,9\right\},
    \end{equation}
    and $J = \lceil 2 \log(P 3^p)/\log(2)\rceil$.
    Then with probability at least $1-14e^{-\frac{\alpha^2 \gamma c_3 \tilde{N}}{3^{p+1} s_0}}$,
    \begin{equation}
        \|\hat{\bm{c}}-\bm{c}^*\|_2 \le
            \left(\frac{4}{1-4\alpha} + 1\right) \min_{\bm{c}\in S} \|\bm{c}^*-\bm{c}\|_2 + \frac{2\|\bm{\eta}\|_2}{1-4\alpha} + \frac{\epsilon}{1-4\alpha} + \frac{16 \delta}{1-4\alpha},
    \end{equation}
    where $\epsilon$ defines how close $\hat{\bm{c}}$ is to the optimal value over $\bm{c} \in S$, i.e.,
    \begin{equation}
        \epsilon \ge \|\bm{u} - A\hat{\bm{c}}\|_2 - \min_{\bm{c} \in S} \|\bm{u} - A \bm{c}\|_2.
    \end{equation}
\end{theorem}

\noindent The result in Theorem~\ref{thm:1} relies on assuming the true coefficient vector $\bm{c}^*$ is close to the range of the generative model, up to a sparse deviation, and that Algorithm~\ref{alg} provides a close to optimal solution.

We will use Lemma~\ref{lemma:bora-result} to prove Theorem~\ref{thm:1} by showing that the Legendre measurement matrix satisfies the S-REC with high probability. This result is given in the following lemma.

\begin{lemma}\label{lemma:SREC}
    Let $\Phi = \frac{1}{\sqrt{N}} \Psi D_{\bm{\xi}} \in\mathbb{R}^{N\times P}$ where $\Psi$ is the Legendre measurement matrix.
    Let $\Omega$, $G$, $S$, $g$, and $s_0$ be as defined in Theorem~\ref{thm:1}.

    Let $\alpha < 1/4$ and $\delta < 4Lr/10$ and suppose $P$ is large enough such that
    \begin{equation}\label{eq:s0bound}
        s_0 > \frac{1}{c_0 c_1 c_3}\left(\log(5) + 4Jk + g\right)
    \end{equation}
    and $N$ is large enough, such that,
    \begin{equation}\label{eq:Nbound-linear}
        N \ge \frac{3^{p}s_0}{\gamma c_3} g.
    \end{equation}
    Then $\Phi$ satisfies the S-REC$(S,1-4\alpha,16\delta)$ with probability at least $1-7e^{-\frac{\alpha^2 \gamma c_3 \tilde{N}}{3^{p+1} s_0}}$, where $\tilde{N} = \min\{N,N_{\epsilon_{max}}\}$ with $N_{\epsilon_{max}}$ as given by (\ref{eq:Neps}) with
    \begin{equation}\label{eq:emax}
        \epsilon^2_{max} := \frac{1}{g}(\log(5)+4Jk)+1,
    \end{equation}
    and $J = \lceil 2 \log(P 3^p)/\log(2)\rceil$.
\end{lemma}

To prove this result we first show that $\Phi$ satisfies a concentration inequality similar to (\ref{eq:A-conc-bound}). Since, to the authors' knowledge, this is the first time such a concentration inequality has been derived for the Legendre measurement matrix, we present these results in the following lemma.

\begin{lemma}\label{lemma:gauss-result-1}
    Let $\Phi = \frac{1}{\sqrt{N}} \Psi D_{\bm{\xi}} \in \mathbb{R}^{N\times P}$ where $\Psi$ is the Legendre polynomial measurement matrix and let $s_0$ be as given implicitly by (\ref{eq:s0}).
    Define the functions
    \begin{equation}\label{eq:fn-1-a}
    \begin{aligned}
        f_1(N) :&= \frac{c_0 c_3 N}{3^p \log^2(s_0) \log^2(P)}, \\
        f_2(N) :&= \gamma \log(P) \sqrt{\frac{c_3}{3^p}}\left(
            \log\left(\frac{N}{c_4^2\gamma^2\log^6(P)}\right)
            + 2\right),
    \end{aligned}
    \end{equation}
    and
    \begin{align}
        \epsilon_{min}^2 &:= \frac{1}{N} c_1 3^p 2 \log^2(2) \log^2(P), \label{eq:emin} \\
        \epsilon_0^2 &:= \frac{1}{N} c_1 3^p s_0 \log^2(s_0) \log^2(P).
    \end{align}
    Then, for any $\bm{x} \in \mathbb{R}$,
    \begin{equation}\label{eq:Tail-Bounds}
        \mathbb{P}\left(
            \left|\|\Phi \bm{x}\|_2^2 - \|\bm{x}\|_2^2\right| \ge \epsilon \|\bm{x}\|_2^2
        \right)
        \le \begin{cases}
            5e^{-f_1(N) \epsilon^2} & \epsilon_{min} \le \epsilon \le \epsilon_0, \\
            6e^{-f_2(N) \epsilon} & \epsilon_0 < \epsilon.
        \end{cases}
    \end{equation}
\end{lemma}

The proofs of these results will be given in Section~\ref{sec:proofs}.

\subsection{Theoretical results for exponential/algebraic decay model}


We study the specific case where $G$ is given by the exponential decay model presented in Section~\ref{sec:methods-genmod}; see (\ref{eq:G}). Recall that Theorem~\ref{thm:1} requires the domain of $G$ be bounded, i.e. $G:\Omega \rightarrow \mathbb{R}^P$ where $\Omega \subseteq B_k(r)$. However, in the case where the generative model $G$ takes the form
\begin{equation}
    \begin{aligned}
        G(\bm{z})   &= [G_1(\bm{z}),G_2(\bm{z}),\dots,G_P(\bm{z})]; \\
        G_i(\bm{z}) &= e^{-\left(\bm{b}^{(i)}\right)^T\bm{z}},
    \end{aligned}
\end{equation}
where $\bm{b}^{(i)} > 0$, we can prove a stronger result. Specifically, rather than requiring that $\bm{z} \in B_k(r)$, we can instead consider any $\bm{z}$ such that $z_j \in [z_j^{(0)},\infty)$. We do this by introducing a change of variables. That is, for $z_j \in [z_j^{(0)},\infty)$, there is an $a_j \in [0,1)$ such that
\begin{equation}
    z_j = \frac{a_j}{1-a_j} + z^{(0)}_{j}.
\end{equation}
This allows us to instead consider $G$ as a function of $\bm{a}$ where $\bm{a} \in [0,1)^k$. In the following lemma we show that $G(\bm{z}(\bm{a}))$ is Lipschitz continuous with respect to $\bm{a}$ and derive the Lipschitz constant. This result applies to Theorem~\ref{thm:1} with $\Omega=[0,1)^k$ and $r=\sqrt{k}/2$. This value of $r$ is sufficient because the $\ell_2$ ball, $B_k(\sqrt{k}/2)$, fully encloses the $[0,1]^k$ cube when appropriately centered.

\begin{lemma}\label{lemma:G-Lip}
    Consider the function $G:\mathbb{R}^k \rightarrow \mathbb{R}^P$ such that
    \begin{equation}
        G_i(\bm{z}(\bm{a}))
            = c_i e^{-(\bm{b}^{(i)})^T \bm{z}(\bm{a})}
            = c_i \prod_{j=1}^k e^{
                -b_j^{(i)} \left(\frac{a_j}{1-a_j} + z^{(0)}_{j}\right)
            }, \quad \quad i=1,\dots,P,
    \end{equation}
    where $c_i \in \mathbb{R}$, $\bm{b}^{(i)} \ge \bm{0}$ for $i=1,\dots,P$, and $\bm{z}^{(0)} \in \mathbb{R}^k$. For $\bm{a} \in [0,1)^k$, $G(\bm{z}(\bm{a}))$ is $L$-Lipschitz continuous with
    \begin{equation}
        L =  \sqrt{Pk} \max_{i,j} L_{i,j},
    \end{equation}
    where
    \begin{equation}
        L_{i,j} = \begin{cases}
            |c_i| b^{(i)}_{j} e^{
                -(\bm{b}^{(i)})^T \bm{z}^{(0)}
            } g\left(b_{j}^{(i)}\right)         & b_{j}^{(i)} > 0, \\
            0                                   & b_{j}^{(i)} = 0,
        \end{cases}
    \end{equation}
    and
    \begin{equation}
        g(b) = \begin{cases}
            1                    & \text{if } b\ge 2, \\
            \frac{4}{b^2}e^{b-2} & \text{if } 0 < b < 2.
        \end{cases}
    \end{equation}
\end{lemma}

In the context of the exponential decay generative model given by (\ref{eq:G}), the constants in Lemma~\ref{lemma:G-Lip} are as follows:
\begin{equation}
    \begin{aligned}
        c_i            &= \frac{|\bm{\alpha}^{(i)}|!}{\bm{\alpha}^{(i)}!}\le 1; \\
        \bm{b}^{(i)}_j &= \begin{cases}
            1                       & j=1\\
            \alpha^{(i)}_j          & j=2,\dots,d+1 \\
            \log(1+\alpha^{(i)}_j)  & j=d+2,\dots,2d+1.
        \end{cases}
    \end{aligned}
\end{equation}
If we only allow exponential and algebraic decay (i.e. $\bm{z}^{(0)} = \bm{0}$), using Lemma~\ref{lemma:G-Lip}, we have that, for the Legendre system with $p\ge 2$,
\begin{equation}
    L \le  p\sqrt{Pk}.
\end{equation}

\section{Proofs}\label{sec:proofs}

In this section we prove the main result. We start with a corollary that combines the RIP and JL results given in Corollary~\ref{corollary:RIP} and Theorem~\ref{thm:JL}.

\begin{corollary}\label{corollary:RIP-plus-JL}
    Let $\Phi=\frac{1}{\sqrt{N}} \Psi D_{\bm{\xi}} \in \mathbb{R}^{N\times P}$ where $\Psi$ is the Legendre measurement matrix. Pick $\epsilon>0$ and suppose $s$ is an even integer such that
    \begin{equation}\label{eq:Nbound}
        N \ge c_1 3^p \epsilon^{-2} s \log^2 (s) \log^2 (P),
    \end{equation}
    where $c_1 = 16 C$ and $C$ is as given in Corollary~\ref{corollary:RIP}. Then for $\bm{x} \in \mathbb{R}^P$ and $\gamma>0$
    \begin{equation}\label{eq:RIP-JL-Bound}
        \mathbb{P}\left(
            \left|
                \|\Phi \bm{x}\|_2^2 - \|\bm{x}\|_2^2
            \right| \ge \epsilon \|\bm{x}\|_2^2
        \right)
        \le 5e^{-\min\left\{c_0 s,\gamma\log^2(s)\log^2(P)\right\}}.
    \end{equation}
\end{corollary}

\begin{proof}
    From Corollary~\ref{corollary:RIP} and the bound given by ({\ref{eq:Nbound}}), we have that $\frac{1}{\sqrt{N}} \Psi$ satisfies the RIP($s$,$\epsilon/4$) with probability
    \begin{equation}
        \mathbb{P}\left(\text{RIP}\left(s,\frac{\epsilon}{4}\right)\right)
        \ge 1 - e^{-\gamma\log^2(s)\log^2(P)}.
    \end{equation}
    Additionally, from Theorem~\ref{thm:JL} and (\ref{eq:JL-given-RIP}), given $\frac{1}{\sqrt{N}} \Psi$ satisfies the RIP($s,\epsilon/4$), we have that $\Phi$ satisfies the JL distributional property at level $\epsilon$, i.e., JL($\epsilon$), with probability
    \begin{equation}
        \mathbb{P}\left(
            \text{JL}(\epsilon) \mid \text{RIP}\left(s,\frac{\epsilon}{4}\right)
        \right)
        \ge 1-4e^{-c_0s},
    \end{equation}
    where $c_0=1/40$. We combine these results to obtain the probability that $\Phi$ satisfies the JL distributional property at level $\epsilon$.
    \begin{equation}
        \begin{aligned}
            \mathbb{P}(\text{JL}(\epsilon))
            &\ge \mathbb{P}\left(
                \text{JL}(\epsilon) \mid \text{RIP}\left(s,\frac{\epsilon}{4}\right)
            \right)\mathbb{P}\left(
                \text{RIP}\left(s,\frac{\epsilon}{4}\right)
            \right) \\
            \mathbb{P}\left(
                \left|\|\Phi \bm{x}\|_2^2 - \|\bm{x}\|_2^2\right| \le \epsilon\|\bm{x}\|_2^2 \mid \bm{x} \in \mathbb{R}^P
            \right)
            &\ge (1-4e^{-c_0 s})\left(1-e^{-\gamma\log^2(s)\log^2(P)}\right) \\
            &\ge 1 - 4e^{-c_0 s} - e^{-\gamma\log^2(s)\log^2(P)} \\
            &\ge 1 - 5e^{-\min\left\{c_0 s,\gamma\log^2(s)\log^2(P)\right\}}.
        \end{aligned}
    \end{equation}
\end{proof}


The probability bound given by (\ref{eq:RIP-JL-Bound}) in Corollary~\ref{corollary:RIP-plus-JL} depends on the sparsity order $s$ at which the RIP is satisfied with high probability. This sparsity order is dependent on the number of samples $N$ and the level of the RIP, i.e., $\epsilon/4$. Therefore, we can instead write the bound given by (\ref{eq:RIP-JL-Bound}) as a function of $N$ and $\epsilon$. This will give us the results as stated in Lemma~\ref{lemma:gauss-result-1} and the modified probability bound given by (\ref{eq:Tail-Bounds}).

\begin{proof}[Proof of Lemma \ref{lemma:gauss-result-1}]
    Pick $\epsilon \ge \epsilon_{min}$ and define $s_{\epsilon}$ implicitly as
    \begin{equation}\label{eq:se}
         s_{\epsilon} \log^2(s_{\epsilon}) = \frac{\epsilon^2 N}{c_1 3^p \log^2(P)}.
    \end{equation}
    Let $s$ be the largest even integer such that $s \le s_{\epsilon}$. Since $\epsilon \ge \epsilon_{min}$, we are guaranteed that $s_{\epsilon}\ge 2$ and, hence, $s \ge 2$. Under these conditions, the required bound in Corollary~\ref{corollary:RIP-plus-JL}, given by (\ref{eq:Nbound}), is satisfied, and, therefore, (\ref{eq:RIP-JL-Bound}) holds.

    We will modify {(\ref{eq:RIP-JL-Bound})} by obtaining lower bounds for $s$ and $\log(s)$. First note that, since $s \ge 2 $ is the largest even integer less than $s_{\epsilon}$, we have that
    \begin{equation}\label{eq:s_se}
        \frac{s \log^2(s)}{s_{\epsilon} \log^2(s_{\epsilon})}
        \ge \min_{i\ge2} \frac{i\log^2 i}{(i+2)\log^2(i+2)}
        = \frac{2\log^2(2)}{4\log^2(4)}
        =: c_2.
    \end{equation}
    For larger values of $s$, the value of $c_2$ approaches 1.
    Using (\ref{eq:se}) and (\ref{eq:s_se}), we obtain the following bound
    \begin{equation}\label{eq:s}
        s \log^2(s)
        \ge c_2 s_{\epsilon} \log^2(s_{\epsilon}) =  \frac{c_3 \epsilon^2 N}{3^p \log^2(P)},
    \end{equation}
    where $c_3=c_2/c_1$.

    When $\epsilon \le \epsilon_0$. We have that $s \le s_{\epsilon} \le s_0$ and using (\ref{eq:s}),
    \begin{equation}\label{eq:sbound1}
        s \ge \frac{c_3 \epsilon^2 N}{3^p \log^2(s) \log^2(P)} \ge \frac{c_3 \epsilon^2 N}{3^p \log^2(s_0) \log^2(P)}.
    \end{equation}
    By definition of $s_0$, see (\ref{eq:s0}), we have that $c_0 s \le \gamma \log^2(s)\log^2(P)$ and therefore, using (\ref{eq:sbound1}), the inequality given by (\ref{eq:RIP-JL-Bound}) reduces to
    \begin{equation}\label{eq:RIP-JL-Bound-1}
        \mathbb{P}\left(
            \left|
                \|\Phi \bm{x}\|_2^2 - \|\bm{x}\|_2^2
            \right| \ge \epsilon \|\bm{x}\|_2^2
        \right)
        \le 5e^{-c_0 s}
        \le 5e^{-\epsilon^2 \frac{c_0 c_3 N}{3^p \log^2(s_0) \log^2(P)}}
        = 5e^{-\epsilon^2 f_1(N)},
    \end{equation}
    where $f_1(N)$ is as defined in (\ref{eq:fn-1-a}).

    When $\epsilon > \epsilon_0$, we have that $s_\epsilon > s_0$ and, using the definition of $s_0$, it follows that
    \begin{equation}
        c_0 (s+2) \ge c_0 s_\epsilon > \gamma \log^2(s_\epsilon)\log^2(P) \ge \gamma \log^2(s)\log^2(P).
    \end{equation}
    Therefore, $c_0 s > \gamma \log^2(s)\log^2(P) - 2c_0$, implying that
    \begin{equation}\label{eq:RIP-JL-Bound-2}
        \mathbb{P}\left(
            \left|
                \|\Phi \bm{x}\|_2^2 - \|\bm{x}\|_2^2
            \right| \ge \epsilon \|\bm{x}\|_2^2
        \right)
        \le 5e^{-\gamma\log^2(s)\log^2(P)+2 c_0}
        \le 6e^{-\gamma\log^2(s)\log^2(P)}.
    \end{equation}
    Using (\ref{eq:s}) and the assumption that $s < N < P$, we obtain the follow two bounds
    \begin{equation}\label{eq:sbound2}
        \log^2(s) \ge \frac{c_3 \epsilon^2 N}{3^p \log^2(P) s} \ge \frac{c_3 \epsilon^2}{3^p \log^2(P)}, \quad \quad \quad s \ge \frac{c_3\epsilon^2 N}{3^p \log^4(P)}.
    \end{equation}
    We then write $\log^2(s)=\log(s)\log(s)$ and bound the first and second $\log(s)$ terms using the first and second bound given in (\ref{eq:sbound2}), respectively,
    \begin{equation}
    \begin{aligned}
        \log^2(s) &= \log(s) \log(s)
        \ge \frac{\epsilon}{\log(P)}\sqrt{\frac{c_3}{3^p}}
            \log\left(\frac{c_3\epsilon^2 N}{3^p \log^4(P)}\right) \\
        &= \frac{\epsilon}{\log(P)}\sqrt{\frac{c_3}{3^p}}
            \left(
                \log\left(\frac{N}{c_4^2\gamma^2\log^6(P)}\right) +     \log\left(\frac{c_4^2\gamma^2c_3\epsilon^2\log^2(P)}{3^p}\right)
            \right) \\
        &\ge \frac{\epsilon}{\log(P)}\sqrt{\frac{c_3}{3^p}}
            \left(
                \log\left(\frac{N}{c_4^2\gamma^2\log^6(P)}\right) +     2 - \frac{2 \sqrt{3^{p}}}{c_4 \gamma\sqrt{c_3}\epsilon\log(P)}
            \right) \\
        &= \frac{\epsilon}{\log(P)} \sqrt{\frac{c_3}{3^p}}\left(
            \log\left(\frac{ N}{c_4^2\gamma^2\log^6(P)}\right)
            + 2\right)
            -\frac{2}{c_4\gamma\log^2(P)},
    \end{aligned}
    \end{equation}
    where $c_4>0$ is an arbitrary constant. For the second inequality above, we use that $\log(x^2) = 2\log(x) \ge 2 - 2/x$. We then have that
    \begin{equation}
        \gamma \log^2(s)\log^2(P) \ge \epsilon\gamma \log(P) \sqrt{\frac{c_3}{3^p}}\left(
            \log\left(\frac{ N}{c_4^2\gamma^2\log^6(P)}\right)
            + 2\right)
            -\frac{2}{c_4} = \epsilon f_2(N) - \frac{2}{c_4}.
    \end{equation}
    We apply this inequality to  (\ref{eq:RIP-JL-Bound-2}), to obtain $f_2(N)$ as defined in (\ref{eq:fn-1-a}). Since $c_4$ is an arbitrary constant, we can pick it such that the constant of 6 in front of the exponential given by (\ref{eq:RIP-JL-Bound-2}) is still valid, i.e., $5c^{2c_0+2/c_4} < 6$. This is true if, for example, $c_4=16$.
\end{proof}

We have now obtained a probability tail bound similar to the bound for the random Gaussian matrix, see (\ref{eq:A-conc-bound}). This bound shows us that at lower values of $N$, the exponential decay rate is a linear function of $N$. However, once the value of $N$ surpasses a threshold value the decay rate is only logarithmically dependent on $N$. For a given $\epsilon$, $N_\epsilon$ as given by (\ref{eq:Neps}) is the sample size at which this transition occurs. As we increase the size of our Legendre basis, i.e., as $P$ increase, the value of $N_\epsilon$ also increases. Therefore, if $P$ is large enough, we can assume  that we are in the linear regime of exponential decay.

In our proof showing that  $\Phi$ satisfies the S-REC property, we will find a value of $P$ and hence $s_0$, i.e., see (\ref{eq:s0}), such that for the needed values of $\epsilon$ we will be in this linear decay regime. This will allow us to use the following corollary, which is a simplified version of Lemma~\ref{lemma:gauss-result-1}.

\begin{corollary}\label{cor:tailbound}
    Let $\Phi = \frac{1}{\sqrt{N}} \Psi D_{\bm{\xi}} \in \mathbb{R}^{N\times P}$ where $\Psi$ is the Legendre polynomial measurement matrix and let $s_0$ be as given by (\ref{eq:s0}).
    Define
    \begin{equation}\label{eq:fn}
        f(N;\epsilon) := \frac{c_0 c_3 \tilde{N}}{3^p \log^2(P)\log^2(s_0)}
    \end{equation}
    where $\tilde{N} = \min\{N,N_{\epsilon}\}$ with $N_{\epsilon}$ as given by (\ref{eq:Neps}).

    Then, for any $\bm{x} \in \mathbb{R}$ and $\epsilon>\epsilon_{min}$, where $\epsilon_{min}$ is given by (\ref{eq:emin}),
    \begin{equation}\label{eq:Tail-Bounds-2}
        \mathbb{P}\left(
            \left|\|\Phi \bm{x}\|_2^2 - \|\bm{x}\|_2^2\right| \ge \epsilon \|\bm{x}\|_2^2
        \right)
        \le 5e^{-f(N;\epsilon) \epsilon^2}.
    \end{equation}

\end{corollary}

\begin{proof}
    Let $\epsilon_0$ and $f_1(N)$ be as defined in Lemma~\ref{lemma:gauss-result-1}.
    If $N \le N_{\epsilon}$ then $\epsilon_{min} < \epsilon  < \epsilon_0$ and $f(N,\epsilon)=f_1(N)$. Thus, (\ref{eq:Tail-Bounds-2}) follows directly from Lemma~\ref{lemma:gauss-result-1}. Suppose instead that $N>N_{\epsilon}$ and therefore $s_\epsilon > s_0$ where $s_\epsilon$ is as defined in (\ref{eq:se}). Let $s$ be the largest even integer such that $s < s_{\epsilon}$. We have that
    \begin{equation}
        \frac{s\log^2(s)}{s_0 \log^2(s_0)} \ge \frac{s\log^2(s)}{s_\epsilon \log^2(s_\epsilon)} \ge c_2,
    \end{equation}
    and therefore the logic given by (\ref{eq:s})-(\ref{eq:RIP-JL-Bound-1}) holds with $N \rightarrow N_{\epsilon}$. Since $f(N,\epsilon)=f_1(N_\epsilon)$ this gives us (\ref{eq:Tail-Bounds-2}).
\end{proof}


We next state two lemmas that show the concentration inequality given by Lemma~\ref{lemma:gauss-result-1} implies the S-REC is satisfied. The first lemma is equivalent to Lemma~8.2 in \cite{Bora2017}, and the proof is nearly identical with few modifications. Specifically, we show that the needed concentration inequality is satisfied at the required values of $\epsilon$. We additionally carry the constants through in all the calculations. Because of the similarity with the result from \cite{Bora2017} the proof of this lemma is moved to the appendix.

\begin{lemma}\label{lemma:bora_8.2_modified}
    Let $\Omega \subseteq B^k(r) = \{\bm{z} \in \mathbb{R}^k \mid \|\bm{z}\|_2 \le r\}$ and let $G:\Omega\rightarrow\mathbb{R}^P$ be a $L$-Lipschitz function. Set $S=G(\Omega)$ and let $M$ be a $\delta/L$-net on $\Omega$ such that $\log |M| \le k\log\left(\frac{4L r}{\delta}\right)$. Let $f$ be a positive, increasing function of $N$ and suppose $N$ is large enough such that
    \begin{equation}\label{eq:Nbound1}
        f(N) \ge 3k\log\left(\frac{4 L r}{\delta}\right).
    \end{equation}
    Let $A \in \mathbb{R}^{N \times P}$ and set $J \ge \log(\|A\|)/\log(2)$. Suppose that for $\bm{x}\in\mathbb{R}^p$ and $j=0,\dots,J$
    \begin{equation}
        \mathbb{P}\left(\|A \bm{x}\|_2^2 \ge (1+\epsilon_j)\|\bm{x}\|_2^2\right) \le 5e^{-f(N) \epsilon_j^2},
    \end{equation}
    where $\epsilon_j$ is such that
    \begin{equation}
        \epsilon_j^2 = \frac{1}{f(N)}(\log(5)+4jk) + 1.
    \end{equation}
    For any $\bm{x} \in S$, if $\bm{x}' = \argmin_{\hat{\bm{x}} \in G(M)} \|\bm{x}-\hat{\bm{x}}\|$, we have that $\|A(\bm{x}-\bm{x}')\| \le C \delta$ with probability at least $1 - 2e^{-\frac{f(N)}{3}}$. Here, $C$ is a constant.
\end{lemma}

\begin{remark}
    In the above lemma if we assume $4Lr/\delta > 10$, then we can set the constant $C = 7$.
\end{remark}

The next lemma is analogous to Lemma 2 in \cite{Dhar2018} and Lemma 4.1 in \cite{Bora2017}, and again the corresponding proof is nearly identical. Because of this we place the proof in the appendix. Note that in contrast to the previous proofs, we carry constants through the calculation and use the oblivious subspace embedding result given by Claim~\ref{claim:oblivious} in \ref{sec:appB-OSE}. By using this form of the oblivious subspace embedding result, we obtain slightly different bounds on the required number of measurements as compared with \cite{Dhar2018}.

\begin{lemma}\label{lemma:bora-mod-2}
    Let $\Omega \subseteq B^k(r) = \{\bm{z} \in \mathbb{R}^k \mid \|\bm{z}\|_2 <r\}$ and let $G:\Omega \rightarrow \mathbb{R}^P$ be $L$-Lipschitz. Define the set $S \subseteq \mathbb{R}^P$ such that
    \begin{equation}
        S=\{ G(\bm{z}) + \bm{\nu} \mid \bm{z} \in \Omega, \|\bm{\nu}\|_0 = \ell\}.
    \end{equation}
    Let $A \in \mathbb{R}^{N\times P}$ and pick $\alpha \in (0,1/4)$.
    Set $J \ge \log(\|A\|)/\log(2)$ and suppose that for $\bm{x}\in\mathbb{R}^p$ and $j=1,\dots,J$ there exists a positive function $f(N)$ such that
    \begin{equation}\label{eq:req1}
        \mathbb{P}\left(\|A \bm{x}\|_2^2 \ge (1+\epsilon_j)\|\bm{x}\|_2^2\right) \le 5e^{-f(N) \epsilon_j^2},
    \end{equation}
    where $\epsilon_j$ is such that
    \begin{equation}\label{eq:ej}
        \epsilon_j^2 = \frac{1}{f(N)}(\log(5)+4jk) + 1.
    \end{equation}
    Let $\alpha<1/4$ and let $g=g(k,\ell,P,\delta,\alpha,L,r)$ be as given in (\ref{eq:g}), if (\ref{eq:req1}) holds with $\epsilon_j \rightarrow \alpha$ and
    \begin{equation}\label{eq:fN-bound}
        f(N) \ge g,
    \end{equation}
    then $A$ satisfies the S-REC$(S,1-4\alpha,2(1+C)\delta)$ with probability at least $1-7e^{-\alpha^2\frac{f(N)}{3}}$.
\end{lemma}

Next we combine the results given in Lemma~\ref{lemma:gauss-result-1} and Lemma~\ref{lemma:bora-mod-2} to prove Lemma~\ref{lemma:SREC}. To apply Lemma~\ref{lemma:bora-mod-2}, we need to show that $\Phi$ satisfies the concentration inequality given by (\ref{eq:req1}) for all $\epsilon_j$ and for $\alpha<1/4$ and that the matrix norm of $\|\Phi\|$ is bounded.

\begin{proof}[Proof of Lemma~\ref{lemma:SREC}]

    First we will show that the required tail concentrations inequalities are satisfied. To do this we will use Lemma~\ref{lemma:bora-mod-2} where $f(N)$ is defined as follows:
    \begin{equation}\label{eq:fN2}
        f(N) = \frac{\gamma c_3 \tilde{N}}{3^p s_0} = \frac{c_0 c_3 \tilde{N}}{3^p \log^2(s_0) \log^2(P)},
    \end{equation}
    where $\tilde{N}=\min\{N,N_{\epsilon_{max}}\}$.
    
    We will first show that the bound given by (\ref{eq:fN-bound}) in Lemma~\ref{lemma:bora-mod-2} holds, i.e., $f(N) \ge g$. If $N<N_{\epsilon_{max}}$, (\ref{eq:Nbound-linear}) immediately implies $f(N) \ge g$. Therefore, we need  to show that (\ref{eq:fN-bound}) holds when $N>N_{\epsilon_{max}}$ and thus, $\tilde{N}=N_{\epsilon_{max}}$ and $f(N)=f(N_{\epsilon_{max}})$. Using (\ref{eq:fN2}) and (\ref{eq:Neps}), we have that
    \begin{equation}\label{eq:fNmax}
        f(N_{\epsilon_{max}}) = c_0 c_1 c_3 \epsilon_{max}^{-2} s_0.
    \end{equation}
    Rearranging (\ref{eq:fNmax}) and applying the bound on $s_0$ given by (\ref{eq:s0bound}), we then have that
    \begin{equation}
        s_0 = \frac{\epsilon_{max}^2 f(N_{\epsilon_{max}})}{c_0 c_1 c_3} \ge \frac{1}{c_0 c_1 c_3}\left(\log(5) + 4Jk + g\right).
    \end{equation}
    Simplifying and using the definition of $\epsilon_{max}$ given by (\ref{eq:emax}), we obtain
    \begin{equation}
        \left(\frac{1}{g}(\log(5)+4Jk)+1\right) f(N_{\epsilon_{max}}) \ge \log(5) + 4Jk + g
    \end{equation}
    which gives us
    \begin{equation}
        f(N_{\epsilon_{max}}) \ge \frac{\log(5) + 4Jk + g}{\frac{1}{g}(\log(5)+4Jk+g)} = g.
    \end{equation}
    Therefore, we have that the condition given by (\ref{eq:fN-bound}) in Lemma~\ref{lemma:bora-mod-2} is satisfied.

    We next need to show that the concentration tail inequality given by (\ref{eq:req1}) is satisfied for $\alpha<1/4$ and $\epsilon_j$ as defined in (\ref{eq:ej}) for $j=1,\dots,J$. To do this we will show that, for $j=1,\dots,J$,
    \begin{equation}\label{eq:e-inequal}
        \epsilon_{min} < \alpha < \epsilon_j < \epsilon_{max}.
    \end{equation}
    This then allows us to apply Corollary~\ref{cor:tailbound} to show that (\ref{eq:req1}) holds for $\alpha$ and $\epsilon_j$, $j=1,\dots,J$. That is suppose $\epsilon$ is such that $\epsilon_{min} < \epsilon < \epsilon_{max}$. Let $f(N;\epsilon)$ and $N_\epsilon$ be as defined in Corollary~\ref{cor:tailbound} and (\ref{eq:Neps}). Since $\epsilon < \epsilon_{max}$, we have that $N_{\epsilon_{max}} < N_\epsilon$ and $f(N;\epsilon) \ge f(N)$. Applying Corollary~\ref{cor:tailbound} gives us
    \begin{equation}\label{eq:Tail-Bounds-old2}
        \mathbb{P}\left(
            \left|\|\Phi \bm{x}\|_2^2 - \|\bm{x}\|_2^2\right| \ge \epsilon \|\bm{x}\|_2^2
        \right)
        \le 5e^{-f(N;\epsilon) \epsilon^2} \le 5e^{-f(N)\epsilon^2}.
    \end{equation}
    This logic shows us that given (\ref{eq:e-inequal}) holds, (\ref{eq:req1}) is satisfied for $\alpha$ and $\epsilon_j$, $j=1,\dots,J$.

    The upper bound in (\ref{eq:e-inequal}) follows immediately from
    \begin{equation}\label{eq:emax-old}
        \epsilon^2_{max} := \frac{1}{g}(\log(5)+4Jk)+1 \ge \frac{1}{f(N)}(\log(5)+4jk)+1 = \epsilon_j^2 \ge \alpha^2.
    \end{equation}
    To prove the lower bound of (\ref{eq:e-inequal}) holds, first note, due to the bound on $N$ given by (\ref{eq:Nbound-linear}) and the definition of $s_0$ given by (\ref{eq:s0}), we have that
    \begin{equation}
        N \ge \frac{c_1 3^p}{c_0 c_2} \log^2(s_0) \log^2(P) g \ge \frac{c_1}{\alpha^2} 3^p 2 \log(2) \log^2(P),
    \end{equation}
    where we pull out the $\alpha^2$ term from the $g$ function and note that the remaining multiplicative terms are greater than 1. Therefore, using the definition of $\epsilon_{min}$ given by (\ref{eq:emin}), we have that $\epsilon_{min}<\alpha<\epsilon_j$ for $j=1,\dots,J$.

    To apply Lemma~\ref{lemma:bora-mod-2}, it remains to show that $J \ge \log(\|\Phi\|)/\log(2)$. It is known that Legendre polynomials are bounded as follows (see Lemma~3.3 from \cite{Doostan2011}),
    \begin{equation}
        \text{max}_{\bm{\alpha}\in\Lambda_{d,p}}\|\psi_{\bm{\alpha}}\|_{\infty} = 3^{\frac{p}{2}}
    \end{equation}
    which gives us the following bound for the matrix norm,
    \begin{equation}\label{eq:Psi-norm-bound}
        \|\Psi\|_2 \le \sqrt{NP} \max_{ij} |\Psi_{ij}| \le \sqrt{NP} 3^{p/2}.
    \end{equation}
    Therefore, $\|\Phi\|_2=\left\|\frac{1}{\sqrt{N}}\Psi D_{\bm{\xi}}\right\|_2\le \sqrt{P}3^{\frac{p}{2}}$ and $J = \lceil 2 \log(P 3^p)/\log(2)\rceil \ge \log(\|\Phi\|)/\log(2)$.

\end{proof}

Finally, we combine Lemma~\ref{lemma:bora-result} and \ref{lemma:SREC} to prove the main result.

\begin{proof}[Proof of Theorem~\ref{thm:1}]
    By Lemma~\ref{lemma:SREC}, we have that $\Phi$ satisfies the S-REC($S$,$1-4\alpha$,$16\delta$) with probability at least $1-7e^{-\frac{\alpha^2 f(N)}{3}}$ where $f(N)$ is as defined in (\ref{eq:fN2}).
    Additionally, using Corollary~\ref{cor:tailbound} and the same logic as described in the proof of Lemma~\ref{lemma:SREC}, we can show (\ref{eq:Tail-Bounds-old2}) holds when $\epsilon=3$, i.e.,
    \begin{equation}
        \mathbb{P}\left(
            \left|\|\Phi \bm{x}\|_2^2 - \|\bm{x}\|_2^2\right| \ge 3 \|\bm{x}\|_2^2
        \right)
        \le 5e^{-9f(N;\epsilon)} \le 5e^{-9 f(N)}.
    \end{equation}
    Note that in the statement of Theorem~\ref{thm:1}, we are guaranteed that $\epsilon_{max} \ge 3$.
    Therefore, $\|\Phi \bm{x}\|_2 \le 2\|\bm{x}\|_2$ with probability at least $1-5e^{-9 f(N)}$. We now apply Lemma~\ref{lemma:bora-result} and set the value of $\mu$ in the lemma such that
    \begin{equation}
        \mu = \max\left\{7e^{-\frac{\alpha^2 f(N)}{3}},5 e^{-9 f(N)}\right\} = 7e^{-\alpha^2 \frac{f(N)}{3}}.
    \end{equation}
\end{proof}

\subsection{Results for the exponential decay generative model}

Next we prove Lemma \ref{lemma:G-Lip}, demonstrating that the results of Theorem~\ref{thm:1} can be applied to the exponential decay model (and variations of this model) described in Section~\ref{sec:methods-genmod}.

\begin{proof}[Proof of Lemma \ref{lemma:G-Lip}]
    For notational simplicity we write $G_i(\bm{a}) = G_i(\bm{z}(\bm{a}))$. First we will bound the derivative of $G_i$ with respect to $a_{\ell}$.
    For $\ell$ such that $b_{\ell}^{(i)}=0$ we have that $\partial G_i(\bm{a})/\partial a_{\ell} = 0$ since
    \begin{equation}
        G_i(\bm{a}) = c_i\prod_{j=1}^k e^{-(b_j^{(i)})^T \left(\frac{a_j}{1-a_j} + z^{(0)}_{j}\right)} = c_i\prod_{\substack{j=1 \\j\ne \ell}}^k e^{-(b_j^{(i)})^T \left(\frac{a_j}{1-a_j} + z^{(0)}_{j}\right)}.
    \end{equation}
    Therefore, below we consider $\ell$ such that $b_{\ell}^{(i)} > 0$. For $a_{\ell} \in [0,1)$,
    \begin{equation}\label{eq:Lip-deriv-1}
    \begin{aligned}
        \left|\frac{\partial G_i(\bm{a})}{\partial a_{\ell}}\right|
            &= \left|\frac{\partial G_i(\bm{a})}{\partial z_{\ell}}\frac{dz_{\ell}}{da_{\ell}}\right| \\
            &= |c_i| b^{(i)}_{\ell} \left(\prod_j e^{-b_j^{(i)} \left(\frac{a_j}{1-a_j}+z_{j}^{(0)}\right)}\right) \frac{1}{(1-a_{\ell})^2} \\
    &= \left(\prod_{j\ne \ell} e^{-b_j^{(i)} \left(\frac{a_j}{1-a_j}+z_{j}^{(0)}\right)}\right) \frac{|c_i| b^{(i)}_{\ell}}{(1-a_{\ell})^2}e^{-b_{\ell}^{(i)} \left(\frac{a_{\ell}}{1-a_{\ell}}+z_{\ell}^{(0)}\right)} \\
    &\le \left(\prod_{j\ne \ell} e^{-b_j^{(i)} z_{j}^{(0)}}\right) \left(|c_i| b^{(i)}_{\ell} e^{-b_{\ell}^{(i)} z_{\ell}^{0}} \right)\left(\frac{1}{(1-a_{\ell})^2}e^{-b_{\ell}^{(i)} \left(\frac{a_{\ell}}{1-a_{\ell}}\right)}\right) \\
    &= |c_i| b^{(i)}_{\ell} e^{-(\bm{b}^{(i)})^T \bm{z}^{(0)}} \left(\frac{1}{(1-a_{\ell})^2}e^{-b_{\ell}^{(i)} \left(\frac{a_{\ell}}{1-a_{\ell}}\right)}\right).
    \end{aligned}
    \end{equation}

    Let us examine what the following term looks like for $a \in [0,1)$ and $b>0$
    \begin{equation}
        f(a) = (1-a)^2 e^{\frac{ba}{1-a}}.
    \end{equation}
    Notice that $f(a)$ is the inverse of the final term in parentheses in (\ref{eq:Lip-deriv-1}). To obtain a bound for $f(a)$, we consider the derivative
    \begin{equation}
        \begin{aligned}
            f'(a)
                &= -2(1-a) e^{\frac{ba}{1-a}} + (1-a)^2 \left(\frac{b}{1-a}+\frac{ba}{(1-a)^2}\right) e^{\frac{ba}{1-a}} \\
                &= (-2 + 2a + b)e^{\frac{ba}{1-a}}.
        \end{aligned}
    \end{equation}
    If $b\ge2$ we have that $f'(a)\ge 0$ for $a \in [0,1)$ and
    \begin{equation}
        f(a) \ge f(0) = 1.
    \end{equation}
    Otherwise $f(a)$ is minimized at $a=(2-b)/2$ and, therefore,
    \begin{equation}
        f(a) \ge f\left(\frac{2-b}{2}\right) = \left(\frac{b}{2}\right)^2 e^{2-b}.
    \end{equation}

    Taken together, this implies that
    \begin{equation}
        \left|\frac{\partial G_i(\bm{a})}{\partial a_{\ell}}\right| \le |c_i| b^{(i)}_{\ell} e^{-(\bm{b}^{(i)})^T \bm{z}^{(0)}} g\left(b_{\ell}^{(i)}\right) =: L_{i,\ell},
    \end{equation}
    where
    \begin{equation}
        g(b) =
            \begin{cases}
                1 & \text{if } b\ge 2 \\
                \frac{4}{b^2}e^{b-2} & \text{if } b \in [0,2).
            \end{cases}
    \end{equation}
    Recall that we are only considering $\ell \in 1,\dots,k$ for which $b_{\ell}^{(i)} \ne 0$. For $\ell$ such that $b_{\ell}^{(i)} = 0$, we define $L_{i,\ell} := 0$ since $\partial G_i(\bm{a})/\partial a_{\ell} = 0$.

    Next, we use the bounds on the partial derivatives to calculate the Lipschitz constant. We have that
    \begin{equation}
        \begin{aligned}
        |G_i(\bm{a}^{(1)}) - G_i(\bm{a}^{(2)})|
            &\le \sum_{i=1}^k L_{i,j} |a_j^{(1)} - a_j^{(2)}| \\
            &\le \max_j L_{i,j} \|a^{(1)} - a^{(2)}\|_1 \\
            &\le \sqrt{k} \max_j L_{i,j} \|a^{(1)} - a^{(2)}\|_2.
        \end{aligned}
    \end{equation}
    Squaring and summing across all $P$ terms gives us
    \begin{equation}
        \begin{aligned}
            \sum_{i=1}^P (G_i(\bm{a}^{(1)}) - G_i(\bm{a}^{(2)}))^2
                &\le \sum_{i=1}^P k (\max_j L_{i,j})^2 \|a_j^{(1)} - a_j^{(2)}\|_2^2 \\
                &\le (\max_{i,j} L_{i,j})^2 P k \|a_j^{(1)} - a_j^{(2)}\|_2^2.
    \end{aligned}
    \end{equation}

    The square root of this equation gives us the $\ell 2$ bound
    \begin{equation}
        \|G(\bm{a}^{(1)}) - G(\bm{a}^{(2)})\|_2 \le \sqrt{Pk}(\max_{ij} L_{i,j}) \|\bm{a}^{(1)} - \bm{a}^{(2)}\|_2.
    \end{equation}
    Therefore, the Lipschitz constant is
    \begin{equation}
        L = \sqrt{Pk} \max_{i,j} L_{i,j}.
    \end{equation}

\end{proof}

\section{Numerical Results}\label{sec:numerical-results}

We examine the performance of the GenMod algorithm on three example problems. The examples are presented in order of increasing complexity, as measured by the random input space dimension. In addition to GenMod, we consider results when no sparse vector is added (i.e., GenMod-NoSparse). In the context of Algorithm~\ref{alg} this is equivalent to setting \textit{max\_iteration}=1 and skipping the weighted lasso step (i.e., skipping lines 22-25). We compare the results of our algorithm with techniques that promote sparsity of the coefficient vector, i.e., OMP and IRW-Lasso (see Algorithm~\ref{alg:IRW-Lasso}). For each replication, we obtain the QoI $u$ at $N$ realizations of the random input vector $\bm{Y}$ and run Algorithm~\ref{alg} with $N_{op} = 4N/5$ and $N_{va}=N/5$. We use the same $N$ data points to perform OMP and Lasso as described in Section~\ref{sec:methods-las-omp}.

To compare the performance of the four approaches, we calculate the relative coefficient and/or reconstruction errors. Specifically, for Examples 1 and 2 we find the relative coefficient error by obtaining the least squares coefficient solution $\bm{c}_{ls}$, using $N_{ls} \gg P$ datapoints. The relative coefficient error is defined as
\begin{equation}
    \varepsilon_{c}(\text{Method})
    := \frac{\|\hat{\bm{c}}(\text{Method})-\bm{c}_{ls}||_2}{\|\bm{c}_{ls}\|_2},
\end{equation}
where $\hat{\bm{c}}(\text{Method})$ is the coefficient vector obtained using a specific method (i.e., GenMod, GenMod-NoSparse, OMP, IRW-Lasso).
For all three examples we report the relative reconstruction error, which quantifies how well the optimized coefficients predict the value of the QoI $u$ for testing data that was not used during training. We let $N_{te}$ denote the number of samples used for testing and note that, for a given sample replication, the same testing data was used for each method. Let $\Psi_{te}\in\mathbb{R}^{N_{te}\times P}$ represent the Legendre measurement matrix evaluated at the $N_{te}$ realizations of $\bm{Y}$ used for testing. Let $\bm{u}_{te}$ represent the vector containing the value of the QoI at each of the $N_{te}$ realizations. The relative reconstruction error is then given as
\begin{equation}
    \varepsilon_u(\text{Method})
    := \frac{\|\Psi_{te} \hat{\bm{c}}(\text{Method})
        - \bm{u}_{te}||_2}{\|\bm{u}_{te}\|_2}.
\end{equation}

We will additionally examine the relative improvement of GenMod compared with OMP or IRW-Lasso. For the relative coefficient and reconstruction error, respectively, this improvement will be reported as a percentage as follows
\begin{align}
    \Delta \varepsilon_c
    &:= \frac{
        \varepsilon_c(\text{Method}) - \varepsilon_c(\text{GenMod})
    }{\varepsilon_c(\text{Method})} \cdot 100 \% \\
    \Delta \varepsilon_u
    &:= \frac{
        \varepsilon_u(\text{Method}) -\varepsilon_u(\text{GenMod})
    }{\varepsilon_u(\text{Method})} \cdot 100 \%.
\end{align}
where Method=OMP or IRW-Lasso.

For Examples 1 and 2, we repeated this training and testing process for $n_r=50$ independent replications at each value of $N$. For Example 3, in order to better quantify the relative reconstruction error distribution, we performed $n_r=100$ independent replications.

\subsection{Example 1: 1D Elliptic Equation}\label{sec:Ex1}

We consider the following elliptic equation
\begin{equation}\label{eq:ex1}
    \begin{aligned}
        -\frac{d}{dx}\left(a(x,\bm{Y})\frac{du}{dx}\right) &= 1,
        \quad x \in \mathcal{D}=(0,1) \\
        u(0,\bm{Y})&=u(1,\bm{Y})=0.
    \end{aligned}
\end{equation}
The stochastic diffusion coefficient $a(x,\bm{Y})$ is given by the expansion
\begin{equation}
    a(x,\bm{Y})=\bar{a}+\sigma \sum_{i=1}^d\sqrt{\lambda_i}\phi_i(x)Y_i,
\end{equation}
where $\bar{a}$ is constant and $\sigma$ controls the magnitude of fluctuations from the mean. Here, $\{\lambda_i\}_{i=1}^d$ and $\{\phi_i(x)\}_{i=1}^d$ are, respectively, the $d$ largest eigenvalues and corresponding eigenfunctions of the Gaussian covariance kernel
\begin{equation}
    C(x_1,x_2) = \exp\left(-\frac{(x_1-x_2)^2}{L^2}\right),
\end{equation}
where $L$ is the correlation length. The random vector $\bm{Y}=(Y_1,\dots,Y_d)$ is assumed to be uniformly distributed on $[-1,1]^d$, i.e., $Y_i \sim U(-1,1)$ for $i=1,..,d$.

We numerically calculated the solution to  the 1D elliptic equation (\ref{eq:ex1}) at $d=14$, $L=1/5$, $\bar{a}=0.1$, and $\sigma=0.03$. These parameter values guarantee that all realizations of $a(x,\bm{Y})$ are strictly positive on $\mathcal{D}$. The solution to the PDE is calculated using the Finite Element Method with quadratic elements. A mesh convergence analysis is performed to ensure that spatial discretization errors are insignificant.

Our QoI is the value of $u$ at the center of the 1D domain, i.e. $u(0.5,\bm{Y})$. We approximate the QoI using a PC expansion containing Legendre polynomials up to degree $p=3$. Using (\ref{eq:P}) this implies that the generative function $G:\mathbb{R}^k \rightarrow \mathbb{R}^P$ maps from a $k=2d+1=29$ dimensional space to a $P=680$ dimensional space (i.e., there are 680 coefficients in the PC expansion). For this example, we examine how the different methods perform as we increase the sample size $N$ from 30 to 320. We use $N_{ls}=40000$ points to find the least squares coefficients and $N_{te}=1000$ testing points to calculate the relative reconstruction error.

We find that at a low sample sizes ($N=30,40$), GenMod consistently outperforms both OMP and IRW-Lasso (see Figure~\ref{fig:data1-a} and \ref{fig:data1-b}). At $N=30$, GenMod outperforms OMP and IRW-Lasso for all 50 sample replications (see Figure~\ref{fig:data1-b}, bottom row) and at $N=40$, GenMod outperforms OMP and IRW-Lasso in all but one or two of the 50 sample replications. Additionally, at low sample sizes (i.e., $N=30,40,80$), GenMod has less variability in the error results compared to both OMP and IRW-Lasso (see Figure~\ref{fig:data1-b}, top row), demonstrating the consistency of the GenMod approach.

When comparing the GenMod and the GenMod-NoSparse methods (see Figure~\ref{fig:data1-b}, middle row), we find that at low values of $N$ the two methods perform similarly (i.e., $N=30,40$). However, as $N$ increases, we find that the error of GenMod-NoSparse plateaus, while the error of GenMod continues to decrease. This is as expected since the number of parameters in the GenMod-NoSparse method does not increase with $N$.

\begin{figure}
   \centering
   \includegraphics[]{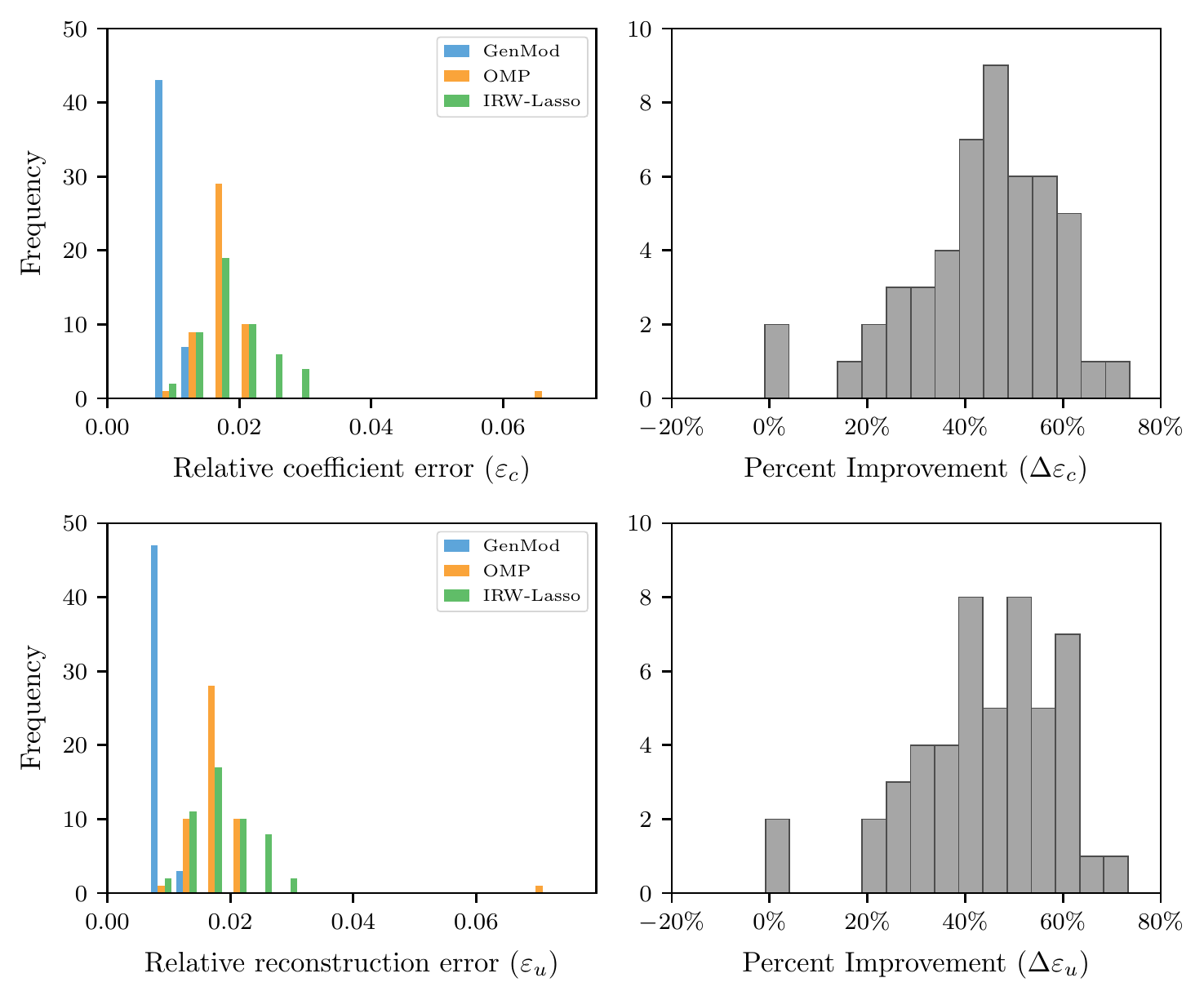}
   \caption{GenMod reduces the coefficient (top row) and reconstruction errors (bottom row) when compared with OMP and IRW-Lasso for Example 1, the 1D elliptic equation. The plots on the right show the percent improvement of GenMod compared with IRW-Lasso. Results are for 50 independent sample replications that used $N=40$ points for training and $N_{te}=1000$ points for testing.}
   \label{fig:data1-a}
\end{figure}

\begin{figure}
    \centering
    \includegraphics[]{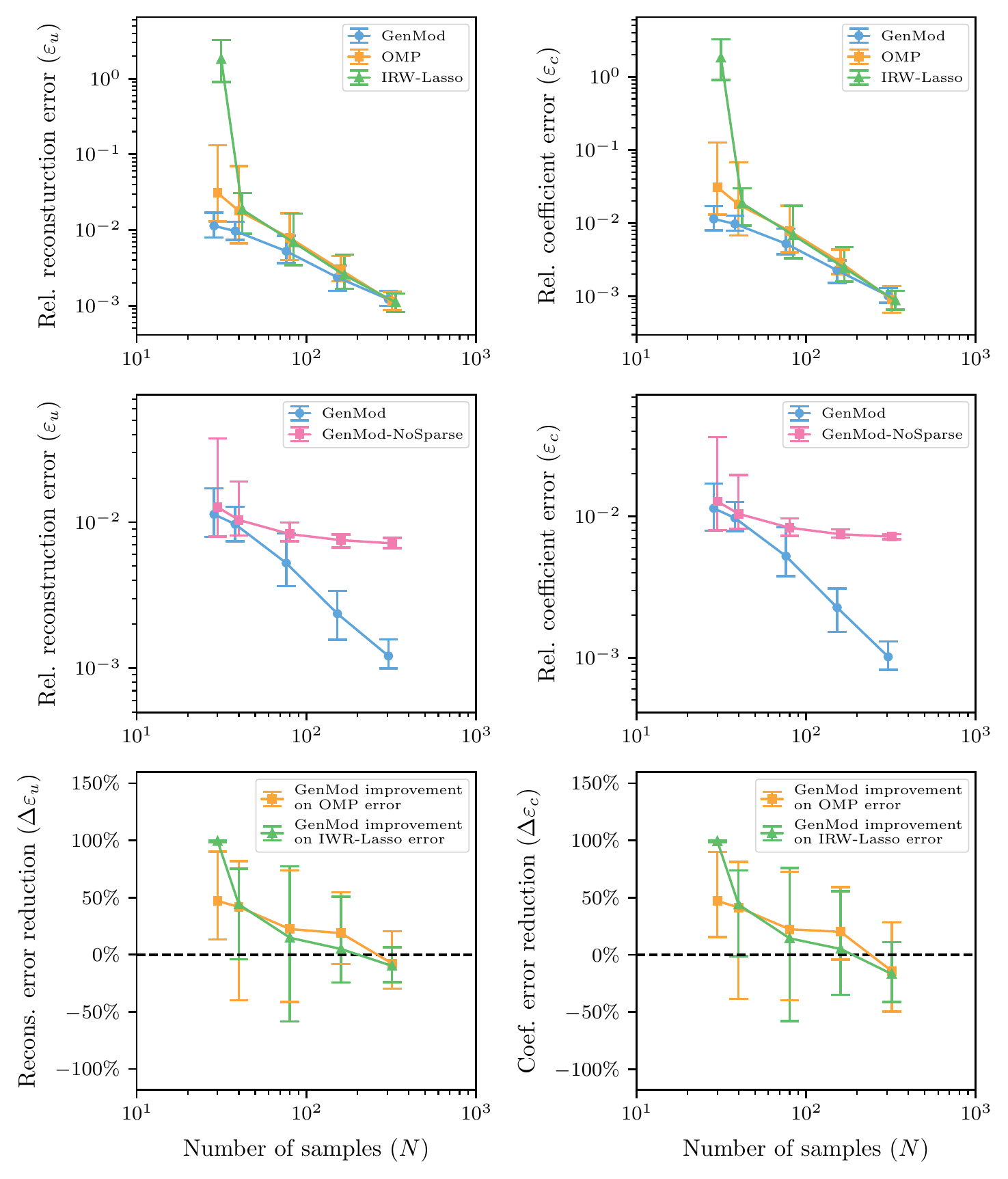}
    \caption{Coefficient and reconstruction errors as a function of sample size $N$ for Example~1, the 1D elliptic equation. Error bars represent the range of results for 50 independent sample replications at each value of $N$. For each replication, $N_{te}=1000$ testing points were used to calculate the relative reconstruction error. Results are given for $N=30,40,80,160,320$. For the distribution of results at $N=40$ see Figure~\ref{fig:data1-a}.}
    \label{fig:data1-b}
\end{figure}

\subsection{Example 2 and 3: Heat driven cavity flows}\label{sec:Ex23}

\begin{figure}
    \centering
    \newcommand{\diagramSize}{5}
    \begin{tikzpicture}
        \draw[->] (-3,0) -- (-3,1) node[above] {$x_2$};
        \draw[->] (-3,0) -- (-2,0) node[right] {$x_1$};

        \filldraw (0,0) circle (1pt) node[below] {\footnotesize (0,0)};
        \filldraw (\diagramSize,\diagramSize) circle (1pt) node[above] {\footnotesize (1,1)};

        \draw[ultra thick,red]  (0,0) -- node[above,rotate=90] {Hot wall} ++(0,\diagramSize);
        \draw[ultra thick,blue]  (\diagramSize,0) -- node[below,rotate=90] {Cold wall} ++(0,\diagramSize);
        \draw[thick]  (0,0) -- node[below] {$\frac{\partial T}{\partial x_2}=0$} ++(\diagramSize,0);
        \draw[thick]  (0,\diagramSize) -- node[above] {$\frac{\partial T}{\partial x_2}=0$} ++(\diagramSize,0);

        \coordinate (A) at (\diagramSize/4,\diagramSize/4);
        \coordinate (B) at (2,2.5);

        \filldraw (A) circle (2pt) node[below] {\footnotesize (0.25,0.25)};
        \draw[-{angle 60},thin] (A) -- node[left] {$\bm{v}$} ++(1,1.5);
        \draw[-{latex},ultra thick] (\diagramSize/4+1,\diagramSize/4) -- node[right] {$v_2$ {\footnotesize $\leftarrow$ Ex~2 QoI}} ++(0,1.5);
        \draw[-{angle 60},thin] (A) -- node[above] {} ++(1,0);

        \filldraw (0,\diagramSize/2) circle (2pt) node[above right,text width=2cm] {{\footnotesize Ex~3 QoI\vspace{-.2cm}
        \hspace{2cm} $\text{   }\downarrow$ \hspace{1cm}}$\partial \Theta/\partial x_1$};
        \draw[-{latex},ultra thick] (0,\diagramSize/2) -- (1.5,\diagramSize/2);

        \draw[blue] (\diagramSize + 1.2,0) to [ curve through={(\diagramSize + 1.4,.2*\diagramSize) . . (\diagramSize + 1,.5*\diagramSize) . . (\diagramSize + 1.4,.6*\diagramSize) . . (\diagramSize + 1,.7*\diagramSize)}] (\diagramSize+1.2,\diagramSize);

        \draw[blue] (\diagramSize + 1,0) to [ curve through={(\diagramSize + .8,0.2*\diagramSize) . . (\diagramSize + 1.6,.6*\diagramSize)}] (\diagramSize + .8,\diagramSize);

        \node[right,blue] at (\diagramSize + 2,\diagramSize/2) {$T_c(x_1=1,x_2,\bm{Y})$};
    \end{tikzpicture}
    \caption{Diagram of the heat driven cavity flow problem. In Example 2 the QoI is the upward velocity $v_2$ at (0.25,0.25). In Example 3 the QoI is the heat flux at the middle of the hot wall.}
    \label{fig:diagram}
\end{figure}
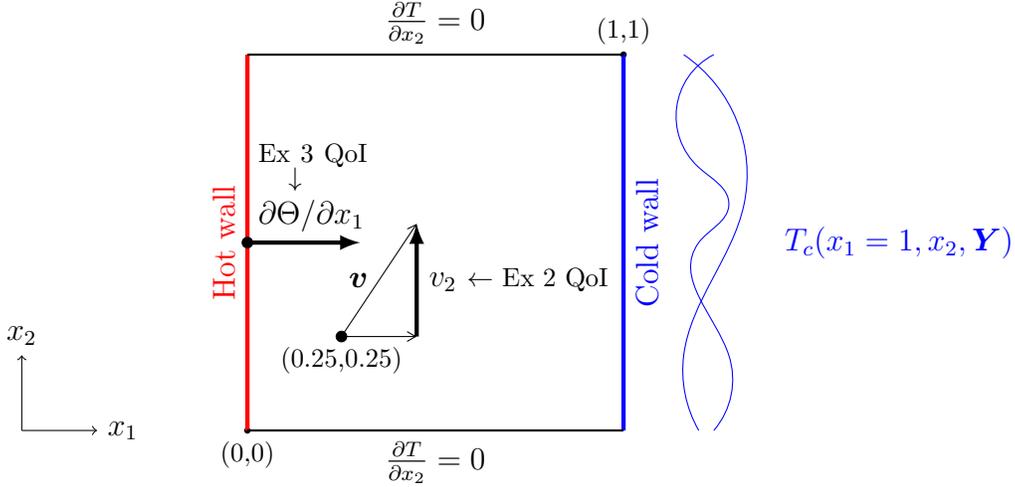

For the next two examples, we consider a 2D heat-driven square cavity flow problem (Figure~\ref{fig:diagram}). This example has been considered previously in \cite{LeMatre2002,Peng2014,Hampton2015,Fairbanks2017,Hampton2018,Hampton2018a}. The left vertical wall has a either a deterministic (Example~2) or random (Example~3) constant temperature $T_h$. For both examples, the right vertical wall has a stochastic, spatially-varying temperature $T_c(x_1=1,x_2,\bm{Y})$ with constant mean $\bar{T}_c$. The top and bottom walls are assumed to be adiabatic, i.e. $\frac{\partial T}{\partial x_2}(x_1,x_2=0,\bm{Y}) = \frac{\partial T}{\partial x_2}(x_1,x_2=1,\bm{Y}) = 0$. Under the assumption of small temperature differences, i.e., the Boussinesq approximation, the governing equations that determine the velocity vector field $\bm{v}=(v_1,v_2)$, the normalized temperature $\Theta$, and the pressure $p$ across the domain are
\begin{equation}
    \begin{aligned}
        \frac{\partial \bm{v}}{\partial t} + \bm{v} \cdot \nabla \bm{v}
        &= -\nabla p + \frac{\text{Pr}}{\sqrt{\text{Ra}}}\nabla^2 \bm{v}
            + \text{Pr} \Theta \bm{e}_{2} \\
        \nabla \cdot \bm{v}
        &= 0 \\
        \frac{\partial \Theta}{\partial t} + \nabla \cdot (\bm{v} \Theta )
        &= \frac{1}{\sqrt{\text{Ra}}}\nabla^2 \Theta,
    \end{aligned}
\end{equation}
where $\bm{e}_{2}$ is the unit vector in the $x_2$ direction, i.e. $\bm{e}_2 = (0,1)$, and $t$ is time. The normalized temperature $\Theta$ is related to the absolute temperature $T$ as follows
\begin{equation}
    \Theta = \frac{T-T_{ref}}{T_h-\bar{T}_c},
\end{equation}
where $T_{ref}$ is the average temperature of the two walls, i.e. $T_{ref} = (T_h + \bar{T}_c)/2$.
This implies the normalized hot and mean cold wall temperatures are $\Theta_h = 0.5$ and $\bar{\Theta}_c = -0.5$, respectively. The dimensionless Prandtl and Rayleigh numbers are defined, respectively, as $\text{Pr}=\nu/\alpha$ and $\text{Ra}=g\tau(T_{h}-\bar{T}_c)L^3/(\nu\alpha)$. Here, $\nu$ is the viscosity, $\alpha$ is thermal diffusivity, $g$ is gravitational acceleration, $\tau$ is the coefficient of thermal expansion, and $L$ is the length of the cavity. For more information on these constants and how they relate to the physical properties of the system see \cite{LeMatre2002}.

On the cold wall, we apply a temperature distribution with stochastic fluctuations as follows
\begin{equation}\label{eq:Ex3-Tc}
    T_c(x_1 = 1,x_2,\bm{Y})
    = \bar{T}_c + \sigma_T\sum_{i=1}^{d_T} \sqrt{\lambda_i}\phi_i(x_2)Y_i,
\end{equation}
or, alternatively, in normalized form,
\begin{equation}\label{eq:Ex2-Thetac}
    \Theta_c(x_1 = 1,x_2,\bm{Y})
    = \bar{\Theta}_c
        + \sigma_{\Theta}\sum_{i=1}^{d_T}\sqrt{\lambda_i}\phi_i(x_2)Y_i,
\end{equation}
where $\sigma_{\Theta} = \sigma_T/(T_h-\bar{T}_c)$. Here, $\sigma_T$ controls the magnitude of the temperature fluctuations along the cold wall, and $\{\lambda_i\}_{i=1}^d$ and $\{\phi_i(x)\}_{i=1}^d$ are, respectively, the $d_T$ largest eigenvalues and corresponding eigenfunctions of the exponential covariance kernel
\begin{equation}
    C(x_1,x_2) = \exp\left(-\frac{|x_1-x_2|}{L^2}\right),
\end{equation}
where $L$ is the correlation length.

Using this physical model as a starting point, we next present details and results for the two examples considered. Note that the QoI and sources of uncertainty differ between these two examples.

\subsubsection{Example 2}\label{sec:Ex2}

In this example, the random input space is composed of the parameters $Y_i$ for $i=1,..,d_T$, which determine the cold wall temperature as given by (\ref{eq:Ex2-Thetac}). We set $d=d_T=20$, $L=1/21$, $\sigma_{\Theta}=11/100$, $\text{Ra}=10^6$ and $\text{Pr}=0.71$, and, by definition we have that $\bar{\Theta}_c=-0.5$ and $\Theta_h=0.5$. Notice that we are setting parameters for the normalized version of the system, as the absolute temperature is not a source of uncertainty.

We will consider polynomial chaos expansions up to order $p=3$. This implies that the generative function $G:\mathbb{R}^k \rightarrow \mathbb{R}^P$ maps from a $k=2d+1=41$ dimensional space to a $P=1771$ dimensional space (i.e., there are 1771 coefficients in the PC expansion). We use $N=60$ points to train the model, $N_{ls}=30000$ points to find the least squares coefficients, and $N_{te}=1000$ points to test and compare the performance of the different methods.

Our results demonstrate that GenMod consistently outperforms both OMP and IRW-Lasso (Figure~\ref{fig:data2-a}). In all but 2 or 4 of the 50 sample replications GenMod decreases the coefficient error compared with OMP and IRW-Lasso, respectively. Similarly, GenMod decreased the reconstruction error in all but 4 of the 50 sample replications for both OMP and IRW-Lasso.

We explored this example further by comparing the absolute values of the coefficients between the three methods (for an example of one replication see Figure~\ref{fig:data2-b}, top row). Lasso often finds nonzero coefficients for polynomials with larger degree that, in reality, do not have significant contributions (e.g., consider the nonzero coefficients with indices over $100$). Meanwhile, OMP often leads to fewer nonzero coefficients and hence misses some coefficients that GenMod accurately predicts (e.g., consider the second coefficient, $\hat{c}_2$). By biasing the expansion towards having decaying coefficients and not enforcing sparsity, GenMod correctly determines that many of the coefficients corresponding to higher degree polynomials are insignificant.

We also examine how well the GenMod algorithm predicts the coefficient signs. For the example replication shown in Figure~\ref{fig:data2-b}, bottom row, we find that, for all coefficients with magnitudes greater than $10^{-5}$, we have correctly predicted the sign. We additionally examine the role of the sparse vector in flipping the signs of the coefficients. We find that there were on average 4.8 sign flips per replication of which 68.1\% were in the correct direction. The example shown in Figure~\ref{fig:data2-b} was chosen to accurately capture the proportion of correct to incorrect sign flips.

\begin{figure}
   \includegraphics{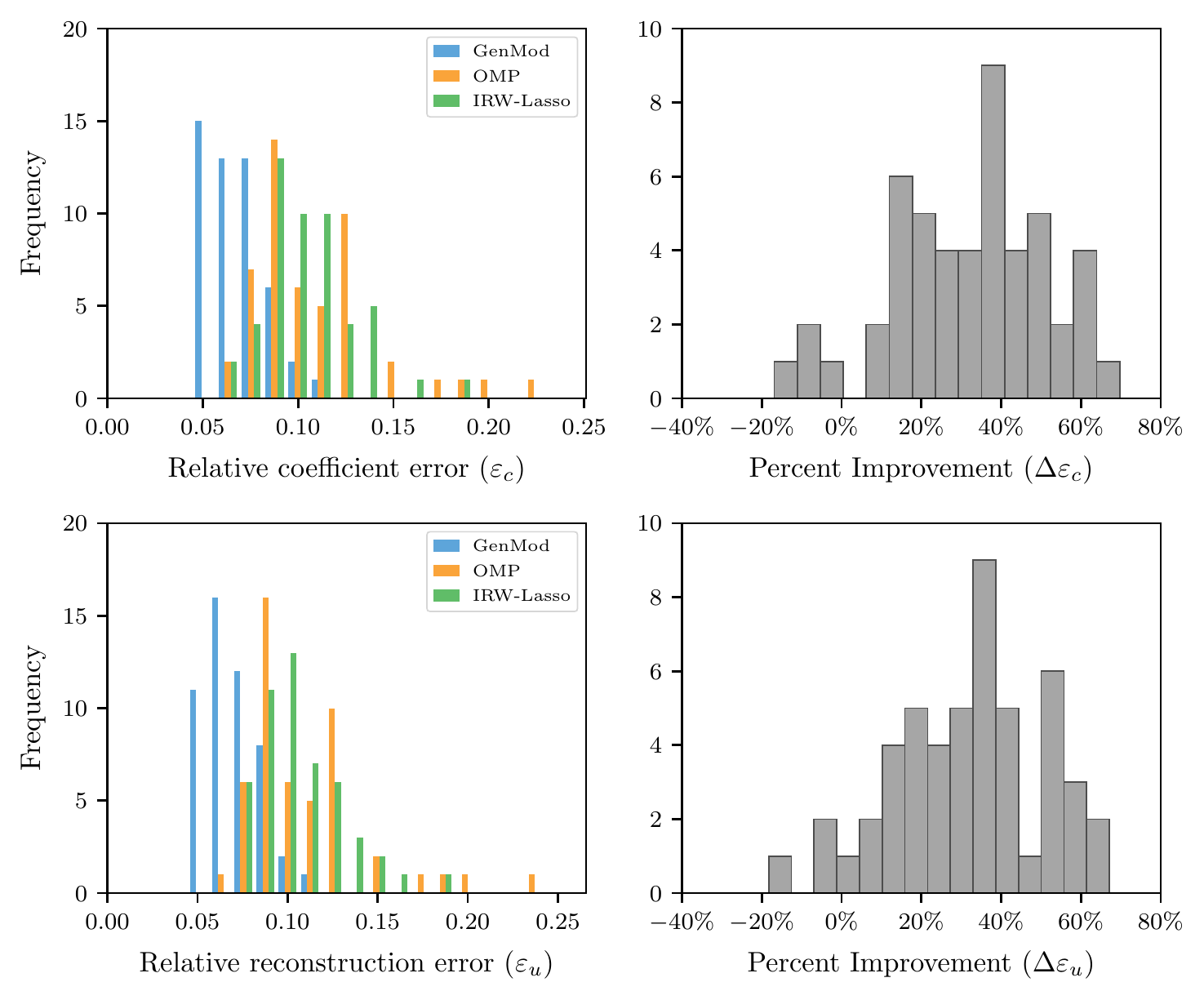}
   \caption{Relative coefficient error (top row) and relative reconstruction error (bottom row) of testing data for Example 2. Results are for 50 independent sample replications that used $N=60$ points for training and $N_{te}=1000$  testing points for calculating the reconstruction error. The right column shows the percent improvement of GenMod compared with IRW-Lasso.}
   \label{fig:data2-a}
\end{figure}

\begin{figure}
    \centering
    \includegraphics{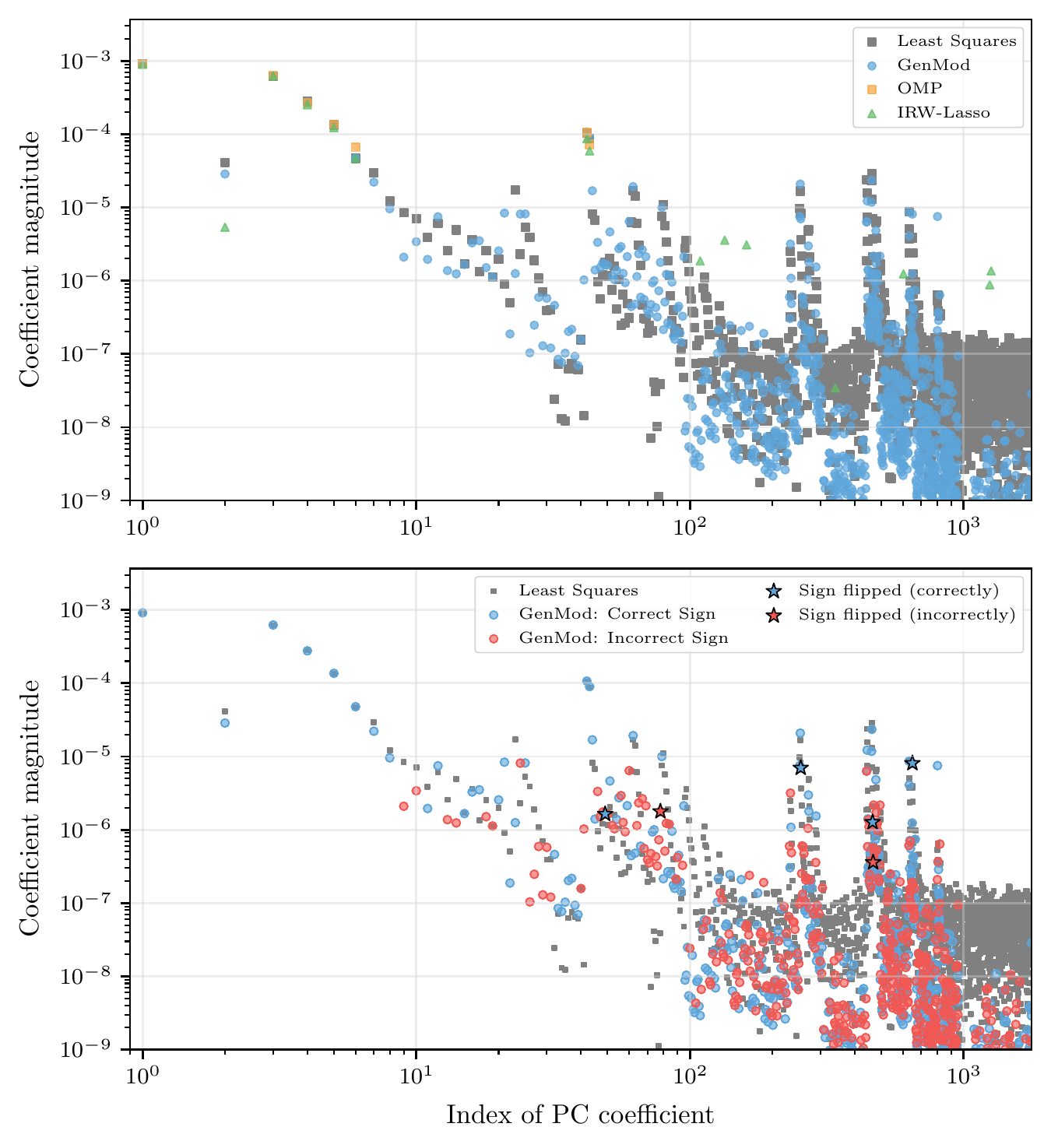}
    \caption{Optimized coefficients for one of the replications in Example 2. The top plot compares the coefficient values obtained for the different optimization methods. The bottom plot shows GenMod coefficients with correct/incorrect signs compared with the least squares optimization. Coefficients with signs that flipped due to the addition of the sparse vector are highlighted with a star.}
    \label{fig:data2-b}
\end{figure}

\subsubsection{Example 3}\label{sec:Ex3}

In this example we increase the size of the random input space by including the hot wall temperature $T_{h}$ and the viscosity $\nu$ as stochastic parameters. We also increase the number of random parameters $d_T$ used to define variations in the cold wall temperature. This example was previously considered in \cite{Hampton2018}. We set $d_T=50$, $\sigma_T=2$, $g=10$, $L=1$, $\tau=0.5$, $\text{Pr}=0.71$, and $\bar{T}_c=100$. Note that here, unlike in Example 2, we are defining the absolute temperature parameters, since one of these parameters $T_h$ is a source of uncertainty. The value of the hot wall temperature is assumed to be uniformly distributed over [105,109] and the value of the viscosity $\nu$ is assumed to be uniformly distributed over [0.004,.01]. Note that these variables can be transformed into corresponding variables $\tilde{T}_h$ and $\tilde{\nu}$  that are uniformly distributed on $[-1,1]$. We then have that the random input vector is $Y=(Y_1,\dots,Y_{d_T}, \tilde{T}_h, \tilde{\nu})$ with size $d=52$

In this example we consider the polynomial chaos expansion up to order $p=2$, implying that the number of coefficients is $P=1431$. Additionally, we have that $k=2d+1=105$. In this example we consider a sample of size $N=50$ and use $N_{te}=250$ points for testing. Note that for this scenario, the sample size $N$ is less than the dimension of the latent space $k$ of the Generative model.

In Example 3, GenMod leads to a decreased variance of the relative reconstruction error (Figure~{\ref{fig:data3}}). More specifically, the mean and standard deviation of the three methods are: $\bar{\varepsilon}_u(\text{GenMod}) = 0.009 \pm 0.0027$, $\bar{\varepsilon}_u(\text{OMP}) = 0.010 \pm 0.0042$, $\bar{\varepsilon}_u(\text{IRW-Lasso})=0.011 \pm 0.0053$. This implies that, although the sparsity promoting methods outperform GenMod for some sample replications (Figure~{\ref{fig:data3}}, left plot), GenMod is more likely to decrease the error below a maximum tolerance level. Additionally, this example shows that GenMod still performs well when the maximum order of the PC expansion, $p$, is small and the number of samples $N$ is less than the dimension of the latent space $k$.

\begin{figure}
    \centering
   \includegraphics{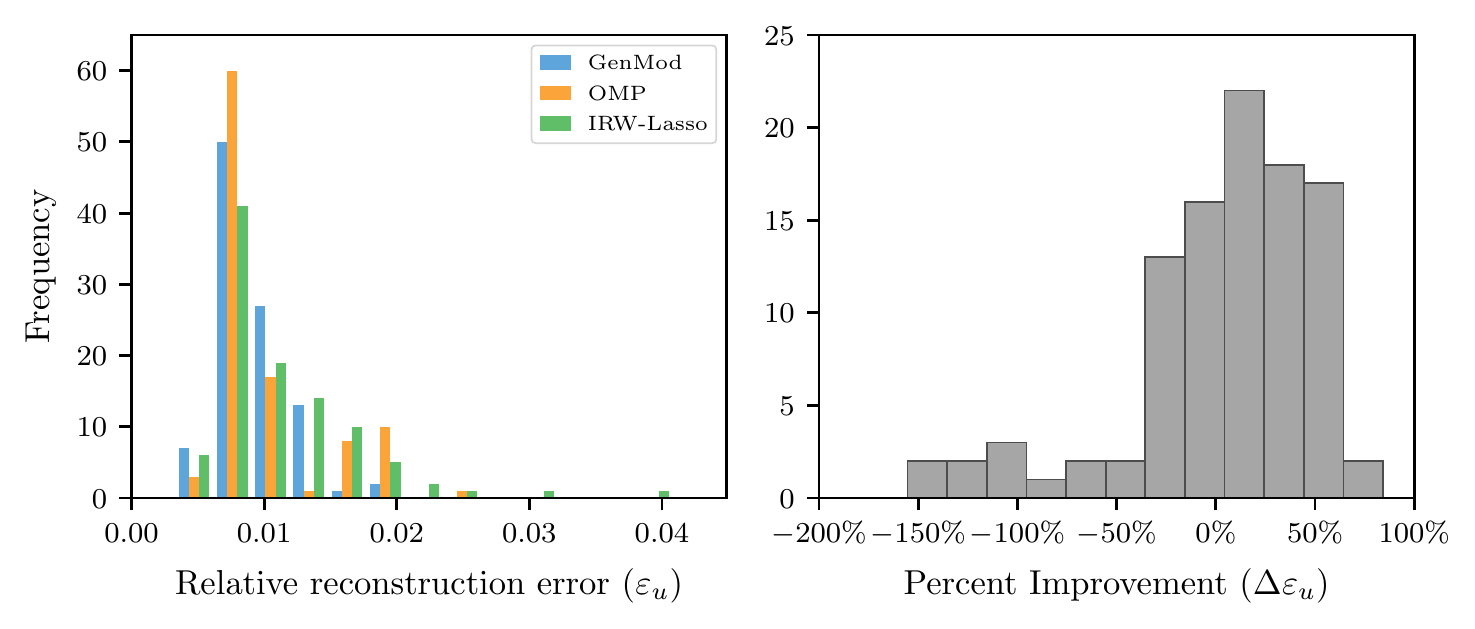}
   \caption{Relative reconstruction error of testing data for Example 3. Results are for 100 independent sample replications that used $N=50$ points for training and $N_{te}=250$ points for validation. The percent improvement of GenMod is shown with respect to IRW-Lasso.}
   \label{fig:data3}
\end{figure}

\section{Discussion}\label{sec:discussion}

We have shown that a nonlinear generative model can improve the prediction of stochastic PDE solutions at small sample sizes, which is the sampling regime of interest in this work. In particular, this work is focused on situations were the number of solution evaluations is roughly equivalent to the number of independent sources of uncertainty. These results have the potential to be further improved by considering alternative generative models and other optimization algorithms.

The generative model we used, see (\ref{eq:G}), is motivated by PDE theory and does not require initial training. Our choice of generative model could be modified in multiple ways. For example, we may explore alternative functions that still encourage exponential decay of the coefficients. The resulting model would still have the quality that initial training is not required. Alternatively, we may consider a generative model that does require initial training. This approach is challenging in systems where measurements are extremely limited. However, by employing ideas from transfer learning or multi-fidelity modeling it might be possible to instead consider a generative model that, for example, is a neural network trained on lower fidelity data. Similar to \cite{Bora2017} our theoretical results apply to a range of generative models and, therefore, are applicable to these alternative scenarios.

We developed the GenMod algorithm (Algorithm~\ref{alg}) to solve the optimization problem given by (\ref{opt:1}), and found that, in practice, this approach performed well. However, this algorithm is likely sub-optimal since it sets the sign of the coefficients independently of the other parameters in the system, i.e., the input to the generative model and the sparse vector. Future work will involve refining this algorithm so that the variables are optimized simultaneously. As another possible improvement, we may develop a way to use the exponential decay function as only a bound on the coefficient magnitude. This would more closely reflect the known PDE theory. For example, at the end of the GenMod algorithm, we may include a relaxation step that allows the coefficients to decrease in magnitude. Determining whether this approach would lead to accurate recovery is a topic of future work.

The generative model approach outperformed OMP and IRW-Lasso for three example physical systems. Notably, alternative sparsity promoting algorithms exist that take into account some of the underlying PC coefficient structure \cite{Baraniuk2010,Peng2014,Chkifa2017}. In future work a more thorough comparative analysis of our approach with these alternative methods is warranted. However, our approach is unique in that we use a nonlinear function to estimate the coefficients and explicitly incorporate the coefficient decay into this function. This work provides a novel method for finding the coefficients in PC expansions for stochastic PDEs.

\section*{Acknowledgment}

The authors acknowledge support by the Department of Energy, National Nuclear Security Administration, Predictive Science Academic Alliance Program (PSAAP) under Award Number DE-NA0003962. The work of AD was also supported by the AFOSR grant FA9550-20-1-0138.

\bibliographystyle{apalike}
\bibliography{library}

\appendix

\section{Additional algorithmic detail}\label{sec:appA}

\subsection{Detail of Adam optimization solver}\label{sec:appA-Adam}

Within the GenMod algorithm (see Algorithm \ref{alg}), we use Adam optimization. For completeness, we present the algorithm for a step of Adam, i.e. AdamStep, in  Algorithm~\ref{alg:AdamStep} (see \cite{Kingma2015} for a more thorough description). In this algorithm we use the analytical gradient, with respect to $\bm{z}$ of the loss function $L(\bm{z},\bm{\nu},\bm{\zeta},\Phi,\bm{u})$ defined by (\ref{eq:L}) where the generative model $G(\bm{z})$ is defined by (\ref{eq:G}).
\begin{algorithm}
    \DontPrintSemicolon
    \SetNoFillComment
    \caption{AdamStep($\bm{z}^{(t-1)}$,$\bm{\nu}$,$\bm{\zeta}$,$\Psi$,$\bm{u}$,$\bm{m}^{(t-1)}$,$\bm{v}^{(t-1)}$,$t$).}
    \label{alg:AdamStep}
    $\bm{g} = \nabla_{\bm{z}} L(\bm{z}^{(t-1)},\bm{\nu},\bm{\zeta},\Psi,\bm{u})$ \;
    $\bm{m}^{(t)} = \beta_1 \bm{m}^{(t-1)} + (1-\beta_1) \bm{g}$ \;
    $\bm{v}^{(t)} = \beta_2 \bm{v}^{(t-1)} + (1-\beta_2) \bm{g}^2$ \;
    $\hat{\bm{m}} = \frac{\bm{m}^{(t)}}{1-\beta_1^t}$ \;
    $\hat{\bm{v}} = \frac{\bm{v}^{(t)}}{1-\beta_2^t}$ \;
    $\bm{z}^{(t)} = \bm{z}^{(t-1)} - \alpha\frac{\hat{\bm{m}}}{\sqrt{\hat{\bm{v}}}+\epsilon}$ \;
    \Return{$\bm{z}^{(t)}$,$\bm{m}^{(t)}$,$\bm{v}^{(t)}$}
    \algorithmfootnote{Adapted from Algorithm 1 of \cite{Kingma2015}. All operations on vectors are element-wise. Throughout the paper we set $\beta_1=0.9$, $\beta_2=0.999$, $\epsilon=10^{-8}$, the values suggested by \cite{Kingma2015}, and the loss function $L$ is given by (\ref{eq:L}).}
\end{algorithm}

\subsection{Detail of iteratively-reweighted Lasso solver}\label{sec:appA-IRW-Lasso}

We next provide the steps of the iteratively-reweighted (IRW) Lasso solver used in this study; see also Section~\ref{sec:methods-las-omp}.

\begin{algorithm}
    \DontPrintSemicolon
    \SetNoFillComment
    \caption{IRW-Lasso($\Phi$,$\bm{u}$).} 
    \label{alg:IRW-Lasso}
    $\tau = \tau_0$ \;
    $\bm{c}^0$ =  LassoWithStErRule($\Psi$, $\bm{u}$; $\bm{\lambda}$) \;
    $k = 1$ \;
    \While{not Converged} {
        $\bm{x}$ = LassoWithStErRule($\Psi (W(\bm{c}^{(k-1)}))^{-1}$, $\bm{u}$; $\bm{\lambda}$)\;
        $\bm{c}^{(k)} = (W(\bm{c}^{(k)}))^{-1}\bm{x}$ \;
        \If{$\|c^{(k)} - c^{(k-1)}\|_2 < \varepsilon$} {
            Converged = True
            }
        \If{$k > \text{max\_iter}$}{
            $k=0$ \;
            $\tau$ = $10\tau$ \;
        }
        \If{$\tau > \tau_{max}$}{
            Break
        }
        $k=k+1$
    }
    \Return{$\bm{c}^{(k)}$}
    \algorithmfootnote{We set $\varepsilon=10^{-6}$ and the vector $\bm{\lambda}$ in the LassoWithStErRule() contains 100 values of which are generated automatically by the scikit learn LassoCV function.}
\end{algorithm}

\section{Additional proofs}\label{sec:appB}

\subsection{Oblivious subspace Embedding}\label{sec:appB-OSE}
In order to prove results when a sparse vector is added to the solution we will need to use a relationship between the JL distributional property and oblivious subspace embeddings. The following claim is a modified version of Lemma 10 from \cite{Sarlos2006}.
\begin{appclaim}\label{claim:oblivious}
    Let $\epsilon$ be such that $0<\epsilon<1$ and suppose a matrix $\Phi \in \mathbb{N\times P}$ satisfies JL($\epsilon$) with probability $p$.
    Let $V$ be an arbitrary $\ell$ dimensional subspace of $\mathbb{R}^P$. Then, for any $\bm{v} \in V$
    \begin{equation}
        \mathbb{P}\left(\left| \|\Phi \bm{v}\|_2^2 - \|\bm{v}\|_2^2\right| \ge 4 \epsilon\|\bm{v}\|_2^2 \right) \le p\left(\frac{\ell}{\epsilon}\right)^{\ell}.
    \end{equation}
\end{appclaim}

\subsection{Proof of Lemma~\ref{lemma:bora_8.2_modified}}

Let $M = M_0 \subseteq M_1 \subseteq M_2 \dots \subseteq M_{J}$ be a chain of epsilon nets of $\Omega$ such that $M_j$ is a $\frac{\delta_j}{L}$-net and $\delta_j = \frac{\delta_0}{2^j}$. We know that there exist a series of nets such that
\begin{equation}
    \log |M_j| \le k\log\left(\frac{4L r}{\delta_j}\right) \le jk + k\log\left(\frac{4L r}{\delta_0}\right).
\end{equation}
Let $N_i = G(M_i)$. Since $G$ is $L$-Lipschitz in $B^k(r)$, the $N_i$'s form a chain of epsilon nets such that $N_i$ is a $\delta_i$-net of $S=G(B^k(r))$ and $|N_i| = |M_i|$.

Pick $\bm{x} \in S$ and let $\bm{x}_0 = \bm{x}'$ be as given in the lemma statement. The exists a sequence of points $\bm{x}_j \in N_j$ for $j=0,\dots,J$ such that
\begin{align*}
    \bm{x} &= \bm{x}_0 + (\bm{x}_1 - \bm{x}_0) + (\bm{x}_2 - \bm{x}_1) + \dots + (\bm{x}_{J} - \bm{x}_{J-1}) \\
    \bm{x} - \bm{x}_0 &= \sum_{j=0}^{J-1}(\bm{x}_{j+1}-\bm{x}_j) + \bm{x}_f,
\end{align*}
where $\bm{x}_f = \bm{x}-\bm{x}_{J}$, $\|\bm{x}_{j+1}-\bm{x}_j\|<\delta_j$ for $j=1,\dots,J$, and $\|\bm{x}-\bm{x}_{J}\|<\delta_{J}$. Using this result we have that
\begin{equation}\label{eq:2terms}
    \|A(\bm{x}-\bm{x}')\| = \|A(\bm{x}-\bm{x}_0)\| \le \sum_{j=1}^{J-1} \|A(\bm{x}_{j+1}-\bm{x}_j)\| + \|A \bm{x}_f\|.
\end{equation}

Therefore, to prove the lemma we need to bound the two terms on the right hand side of (\ref{eq:2terms}) by $\mathcal{O}(\delta)$. First, we consider the $\|A \bm{x}_f\|$ term. Picking $J \ge \log(\|A\|)/\log(2)$, it follows that
\begin{equation}
    \|A \bm{x}_f \| \le \|A\| \|\bm{x}_f\| = \|A\| \frac{\delta_0}{2^{J}} \le \delta_0 = \delta.
\end{equation}

Next, we consider the summation series given on the right hand side of (\ref{eq:2terms}). For $j \in \{0,1,\dots,J-1\}$, let
\begin{equation}
    T_j = \{\bm{x}_{j+1}-\bm{x}_j \mid \bm{x}_{j+1} \in N_{j+1},\bm{x}_j\in N_j\}.
\end{equation}
The size of this set is bounded as follows, for $j=0,\dots,J-1$,
\begin{equation}\label{eq:Tmag}
\log(|T_j|) \le \log(|N_{j+1}|) + \log(|N_j|) \le 3jk + 2k\log\left(\frac{4L r}{\delta_0}\right).
\end{equation}

From the lemma statement we have that for $\bm{t} \in T_j$,
\begin{equation}
    \mathbb{P}\left(\|A \bm{t}\|_2^2 \ge (1+\epsilon_j)\|\bm{t}\|_2^2\right) \le 5e^{-f(N) \epsilon_j^2} = e^{-(f(N)+4jk)}.
\end{equation}
By union bounding over all $\bm{t} \in T_j$ for $j=1,\dots,J-1$ we have
\begin{equation}
    \mathbb{P}\left(\|\Phi \bm{t}\|_2^2 \le (1+\epsilon_j)\|\bm{t}\|_2^2, \forall \bm{t} \in \cup_{j=0}^{J-1} T_j \right) \ge 1 - \sum_{i=0}^{J-1} |T_j| e^{-(f(N)+4jk)}.
\end{equation}

If $f(N) \ge 3k\log\left(\frac{4L r}{\delta_0}\right)$, using (\ref{eq:Tmag}) we have that
\begin{equation}
\begin{aligned}
    \log\left(|T_j|e^{-(f(N)+4jk)}\right)
        &= \log(|T_j|) - f(N) - 4jk \\
        &\le 3jk +  2k\log\left(\frac{4L r}{\delta_0}\right) - f(N) - 4jk \\
        &\le -jk - \frac{f(N)}{3}
\end{aligned}
\end{equation}
and
\begin{equation}
    \sum_{j=0}^{J-1} |T_j|e^{-(f(N)+4jk)} \le e^{-\frac{f(N)}{3}}\sum_{j=0}^{J-1} e^{-jk} \le 2e^{-\frac{f(N)}{3}}.
\end{equation}

Therefore, with probability at least $1-2e^{-\frac{f(N)}{3}}$,
\begin{equation}
    \begin{aligned}
        \sum_{j=0}^{J-1} \|A (\bm{x}_{j+1}-\bm{x}_j)\|
            &\le \sum_{j=0}^{J-1}\sqrt{1+\epsilon_j}\|\bm{x}_{j+1}-\bm{x}_j\| \\
            &\le \sum_{j=0}^{J-1}(1+\epsilon_j)\delta_j \\
            &\le \sum_{j=0}^{J-1}\left(2+\frac{1}{f(N)}(\log 5 + 4jk)\right)\delta_j \\
            &\le \delta_0 \sum_{j=0}^{J-1} \frac{1}{2^j}\left(2+\frac{\log 5}{f(N)}\right) + \delta_0 \sum_{j=0}^{J-1}\frac{4jk}{2^j f(N)} \\
            &\le 2 \delta_0 \left(2+\frac{\log 5+4k}{f(N)}\right) \\
            &\le 2 \delta_0 \left(2+\frac{\log 5+4k}{3k\log\left(\frac{4Lr}{\delta_0}\right)}\right) \\
            &\le 6 \delta_0.
    \end{aligned}
\end{equation}
In the last inequality we assume that $4 L r/\delta_0 > 10$.

\subsection{Proof of Lemma~\ref{lemma:bora-mod-2}}

Pick $\bm{x},\bm{x}' \in S=\{G(\bm{z}) + \bm{\nu} \mid \bm{z} \in \Omega,\|\bm{\nu}\|_0 = \ell\}$ and note that by definition there is then a $\bm{z},\bm{z}' \in \Omega$ and a $\bm{\nu} \in \mathbb{R}^P$ such that $\|\bm{\nu}\|_0 = 2 \ell$ and $\bm{x}-\bm{x}' = G(\bm{z})-G(\bm{z}')+\bm{\nu}$.

We construct a $\delta/L$-net, $M$, on $\Omega$ such that
\begin{equation}
  \log|M|\le k\log\left(\frac{4L r}{\delta}\right).
\end{equation}
By Lipschitz continuity, $G(M)$ is a $\delta$-cover of $S=G(\Omega)$. Define
\begin{equation}
  T = \{G(\bm{z}_1)-G(\bm{z}_2) \mid \bm{z}_1,\bm{z}_2 \in M \}.
\end{equation}
The size of this set is bounded as follows
\begin{equation}
  \log |T| \le 2k\log\left(\frac{4 L r}{\delta}\right).
\end{equation}

For $\bm{z},\bm{z}'$ pick $\bm{z}_1,\bm{z}_2 \in M$ such that $G(\bm{z}_1)$, $G(\bm{z}_2)$ are $\delta$-close to $G(\bm{z})$ and $G(\bm{z}')$, respectively. We have that
\begin{equation}
  \|G(\bm{z}) - G(\bm{z}') + \bm{\nu}\| \le \|G(\bm{z}_1)-G(\bm{z}_2)+\bm{\nu}\| + 2\delta
\end{equation}
and, by Lemma~\ref{lemma:bora_8.2_modified}
\begin{equation}
  \|\Phi G(\bm{z}_1)- \Phi G(\bm{z}_2)+ \Phi \bm{\nu}\|\le \| \Phi G(\bm{z})- \Phi G(\bm{z}')+ \Phi \bm{\nu}\|+ 2 C \delta
\end{equation}
with probability at least $1-2e^{\frac{-f(N)}{3}}$.

Note that $\bm{y} = G(\bm{z}_1) - G(\bm{z}_2) + \bm{\nu}$ lies in a $2\ell +1$ dimensional subspace. Therefore, using the relationship between the JL property and oblivious subspace embeddings given by Claim~\ref{claim:oblivious}, we have that
\begin{equation}
  \label{eq:obv-bound}
  \mathbb{P}\left(\|\Phi \bm{y}\|_2^2 \le (1 - 4\alpha)\|\bm{y}\|_2^2\right) \le  5 \left(\frac{2 \ell + 1}{\alpha}\right)^{2 \ell + 1}e^{-f(N) \alpha^2}.
\end{equation}
Let $E$ denote that set of all possible values of $\bm{y}$. The size of this set is
\begin{equation}
  \log(|E|) \le \log(|T|) + \log\binom{P}{2\ell} \le 2k\log\left(\frac{4Lr}{\delta}\right) + 2 \ell\log\left(\frac{eP}{2 \ell}\right).
\end{equation}
Union bounding (\ref{eq:obv-bound}) over all elements in $E$, we have that
\begin{equation}
  \label{eq:adj_rq}
  \mathbb{P}\left(\|\Phi \bm{y}\|_2^2 \le (1 - 4\alpha)\|\bm{y}\|_2^2,\forall \bm{y} \in E\right) \le 5 |E| \left(\frac{2\ell + 1}{\alpha}\right)^{2\ell + 1}e^{-f(N) \alpha^2}.
\end{equation}
To bound this probability note that,
\begin{equation}
  \begin{aligned}
  &\log\left(|E| \left(\frac{2\ell + 1}{\alpha}\right)^{2\ell + 1}e^{-f(N) \alpha^4}\right) = \log(|E|) + (2\ell+1)\log\left(\frac{2\ell + 1}{\alpha}\right) - f(N)\alpha^2 \\
     &\le 2k \log \left(\frac{4L r}{\delta}\right) + 2\ell\log\left(\frac{eP}{2 \ell}\right) + (2\ell+1)\log\left(\frac{2\ell + 1}{\alpha}\right) - f(N)\alpha^2 \\
     &\le 2k \log \left(\frac{4L r}{\delta}\right) + (2\ell+1)\log\left(\frac{eP(2\ell+1)}{2 \ell \alpha}\right) - 3k\log\frac{4 L r}{\delta} - \frac{3}{2} (2\ell+1)\log\left(\frac{eP(2\ell+1)}{2\ell\alpha}\right)  \\
     &\le -k \log\frac{4 L r}{\delta} - \frac{1}{2} (2\ell+1)\log\left(\frac{eP(2\ell+1)}{2\ell\alpha}\right)  \\
     &= -\frac{\alpha^2}{3} f(N).
  \end{aligned}
\end{equation}
This gives us the following probability bound
\begin{equation}
  \mathbb{P}\left(\|\Phi \bm{y}\|_2^2 \le (1 - 4\alpha)\|\bm{y}\|_2^2,\forall \bm{y} \in E\right)\le 5e^{-\alpha^2 \frac{f(N)}{3}}.
\end{equation}
Taken together and noting that $0 < \alpha < 1$, this implies that with  probability at least $1 - 5e^{-\alpha^2 \frac{f(N)}{3}} - 2e^{-\frac{f(N)}{3}} \le 1 - 7e^{-\alpha^2\frac{f(N)}{3}}$,
\begin{equation}
\begin{aligned}
(1 - 4\alpha)\|G(\bm{z})-G(\bm{z}')+\bm{\nu}\| &\le (1 - 4 \alpha)\|G(\bm{z}_1)-G(\bm{z}_2)+\bm{\nu}\| + 2(1-4\alpha)\delta \\
  &\le \|\Phi(G(\bm{z}_1)-G(\bm{z}_2)+\bm{\nu})\|+2\delta \\
  &\le \|\Phi(G(\bm{z})-G(\bm{z}')+\bm{\nu})\|+2(C+1)\delta.
\end{aligned}
\end{equation}

\end{document}